\documentclass{article}
\usepackage{arxiv}

\usepackage[utf8]{inputenc} 
\usepackage[T1]{fontenc}    
\usepackage{hyperref}       
\usepackage{url}            
\usepackage{booktabs}       
\usepackage{amsfonts}       
\usepackage{nicefrac}       
\usepackage{microtype}      
\usepackage{lipsum}
\usepackage{graphicx}
\graphicspath{ {./images/} }

\usepackage{times}
\usepackage{soul}
\usepackage{url}

\usepackage[small]{caption}
\usepackage{graphicx}
\usepackage{amsmath}
\usepackage{amsthm}
\usepackage{booktabs}
\usepackage{algorithm}
\usepackage{algorithmic}
\usepackage[switch]{lineno}

\usepackage{bm}
\usepackage{amssymb}

\usepackage{threeparttable, array, float} 
\usepackage{color, colortbl}
\usepackage[capitalize]{cleveref}
\usepackage{nicefrac} 
\usepackage{xcolor}
\usepackage{graphicx}
\usepackage{float}
\usepackage{newfloat}
\usepackage{subcaption}
\usepackage{mathtools}
\usepackage{indentfirst}

\definecolor{dark-red}{rgb}{0.4, 0.15, 0.15}
\definecolor{dark-blue}{rgb}{0.15, 0.15, 0.4}
\definecolor{medium-red}{rgb}{0.5, 0, 0}
\definecolor{medium-blue}{rgb}{0, 0, 0.5}
\definecolor{light-red}{rgb}{0.7, 0, 0}
\definecolor{light-blue}{rgb}{0, 0, 0.7}

\definecolor{red}{HTML}{E51400} 
\definecolor{blue}{HTML}{0050EF} 
\definecolor{green}{HTML}{008A00} 
\definecolor{purple}{HTML}{AA00FF} 
\definecolor{orange}{HTML}{FF7F00}
\definecolor{gray}{HTML}{848482}
\definecolor{Gray}{gray}{0.85}
\definecolor{LightGray}{gray}{0.96}

\newtheorem{theorem}{Theorem}[section]

\newtheorem{assumption}{Assumption}[section]
\newtheorem{remark}{Remark}[section]
\newtheorem{lemma}{Lemma}[section]

\title{Scaling Federated Linear Contextual Bandits via Sketching}

\author{
Hantao Yang \\
  University of Science and Technology of China\\
  \texttt{yanghantao@mail.ustc.edu.cn} \\
   \And
Hong Xie\\
  University of Science and Technology of China\\
  \texttt{xiehong2018@foxmail.com} \\
  \And
Xutong Liu\\
  University of Washington Tacoma\\
  \texttt{liuxt2371@gmail.com} \\
  \And
Defu Lian\\
  University of Science and Technology of China\\
  \texttt{liandefu@ustc.edu.cn} \\
}

\begin{document}
\maketitle
\begin{abstract}
In federated contextual linear bandits, high data dimensionality incurs prohibitive computation and communication costs: local agents perform $O(d^3)$-time determinant computation and upload $O(d^2)$ parameters, making existing algorithms unscalable, where $d$ is the dimension of data. To relieve these scaling bottlenecks, this paper proposes Federated Sketch Contextual Linear Bandits (FSCLB). On the computation side, FSCLB uses SVD to indirectly obtain the determinant required for communication, eliminating the prohibitive cost of direct determinant calculation and cutting complexity from $O(d^3)$ to $O(l^2d)$ per round, where $l< d$ is the sketch size. 
On the communication side, FSCLB introduces a double-sketch strategy that reduces both upload and download costs 
from $O(d^2)$ to $O(ld)$. Naively involving sketch update into federated contextual linear bandits can destroy the local increment and invalidate the asynchronous communication condition; FSCLB solves this by replacing the covariance matrix with the sketch matrix when deciding whether to communicate. Theoretically, FSCLB achieves a regret bound of $\widetilde{O} ((\sqrt{d}+\sqrt{M\varepsilon_l})\sqrt{lT})$, where $\varepsilon_l$ is the upper bounded by the spectral tail of the covariance matrix; when $l$ exceeds the rank of the covariance matrix, the bound simplifies to $\widetilde{O}(\sqrt{ldT})$, matching the optimal no-sketch regret. Experiments on both synthetic and real-world datasets show that FSCLB significantly reduces computational and communication costs by over 90 \% while sacrificing only a negligible amount of cumulative reward. 
\end{abstract}


\section{Introduction}
The stochastic multi-armed bandit (MAB) problem is a classic strategy for solving the exploration–exploitation problem and is widely used in recommendation systems \cite{10.5555/944919.944941,Banditalgorithms}. Contextual linear bandits extend MAB by associating each arm with a context vector; the expected reward is linear in this context, governed by an unknown parameter. Strong regret bounds were established in \cite{OFUL}, and the linear bandit model has since been widely deployed in online recommendation \cite{10.1145/1772690.1772758}, web advertising \cite{Tang2013AutomaticAF}, and economics \cite{10.1145/3583681}.

Driven by data scattered across myriad clients and edge devices, federated learning (FL) \cite{pmlr-v54-mcmahan17a,MLSYS2020_1f5fe839} has emerged as the dominant paradigm for distributed machine learning, enabling joint training of a global model without sharing raw data. This privacy-preserving framework has naturally been extended to the bandit setting, giving rise to federated multi-armed bandits (FMAB) \cite{LI2024110663,10.5555/3722577.3722714}. FMAB has inspired federated contextual linear bandits, enabling a swarm of agents to cooperatively solve a global linear bandit while maintaining performance guarantees \cite{FedLinUCB}. 

Local computation and communication overhead directly determine the overall efficiency of federated learning. For the local computation, the ever-growing scale of data makes high-dimensional vectors a major source of prohibitive computation and communication costs, ultimately degrading the algorithmic efficiency. High-dimensional objects can quickly incur prohibitive local costs: a single 10 000 × 10 000-pixel medical image, when unfolded into 32-bit floats, already occupies 400 MB \cite{mathematicalcomputation}. In many real-world applications, the vectors that must be processed can easily reach thousands of dimensions \cite{zhang2025trafficsimulationsmulticitycalibration,KUMARAGE2023102804,OSORIO201918}. In previous work on federated contextual linear bandits, agents usually have to compute a matrix determinant each round to decide whether to communicate with the server; this determinant evaluation incurs $O(d^3)$ complexity, which is prohibitive when the dimension $d$ is large. On the other hands, FL may suffering from
the costly communication. In federated settings, communication efficiency must be considered during on-device training, as clients are often connected to the central server via slow connections ($\approx$ 1Mbps) \cite{FetchSGD}. The high cost of communication has motivated a recent interest in improving communication-efficient protocols for distributed and federated systems \cite{NEURIPS2024_0c433d14}. 

To tackle the above computation and communication overheads, sketching is proposed as an effective dimensionality-reduction technique that has been widely applied in areas such as federated learning \cite{FetchSGD} and low-rank approximation \cite{DBLP:journals/corr/TroppYUC16}. Sketching reduces dimensionality by compressing a large matrix into a compact sketch. \cite{SOFUL} combined the Frequent-Directions \cite{FD} sketch with contextual linear bandits, reducing the computational complexity of contextual linear bandits in high dimensions. 

While sketching mitigates large-scale computation and communication in FL, involving it into FMAB raises new challenges. Classic federated contextual linear bandits rely on covariance matrix growth to trigger uploads; a sketched matrix, however, truncates increments and can mask this growth, rendering the original trigger infeasible.  
Moreover, sketching only cuts the upload cost; if the server naively averages the received sketches for policy computation, the download cost remains $O(d^2)$. Finally, because every sketch is only an approximation, the aggregator must curb error accumulation to retain global accuracy.

In this paper, we proposed federated sketch contextual linear bandits (FSCLB) to reduce the computation cost and communication by involving sketch method. The main contributions of this paper are summarized as follows: 
\\
\textbf{Algorithmic framework.} 

1) \textit{Reducing computation with SVD.} For the computation cost, our algorithm lets each local agent compute the determinant indirectly: it merges the local sketch and server sketch, performs SVD on the resulting small matrix, and obtains the determinant from the singular values. This replaces the determinant calculation with an SVD that fits naturally into the sketch-update pipeline. After introducing SCFD, the SVD is executed only on the $O(l\times d)$ sketch matrix, reducing the original $O(d^3)$ to $O(l^2d)$ with sketch size $l<d$. 

2) \textit{Reducing communication with double sketch:} 
\begin{itemize}
\item  For the upload communication cost, local agents and the server exchange only the compact sketch matrix and the accumulated singular values instead of the full covariance matrix, cutting the per-round payload from $O(d^2)$ to $O(ld)$ ($l<d$ is the sketch size). 
\item For the download communication cost, we propose a double-sketch strategy to limit the communication cost during the download. If the server simply merges all the information, it still has to transmit a $d\times d$ matrix during the download phase. Therefore, after the local side has already sketched the covariance matrix, we perform an additional sketching step on the server once the individual sketch matrices are merged, and the server then carries out all computations using this double-sketched matrix. This keeps the download message aligned with the upload sketch matrix while guaranteeing that the server’s returned sketch matrix and the local sketch matrix share the same lower dimension, ensuring low communication overhead in both the upload and download.
\end{itemize}

3) \textit{Regret guarantee with SCFD.} In our algorithm, we adopt Spectral Compensation Frequent Directions (SCFD)\cite{SCFD} as the sketch-update strategy. SCFD not only accommodates our proposed dual-sketch strategy but also yields a tighter regret bound than the classical FD approach. On the other hand, SCFD is monotonically non-decreasing, which matches the asynchronous-communication rule that uploads occur only after sufficient local growth, ensuring asynchronous communication remains valid.
\\
\textbf{Theoretical analysis.} We show that FSCLB incurs $O(l^2d)$ per-round computation and transmits only $O(ld)$ scalars whenever communication occurs. We prove its regret is bounded by $\widetilde{O}((\sqrt{d}+\sqrt{M\varepsilon_l})\sqrt{lT})$, where $\varepsilon_l$ is upper bounded by the spectral tail of the covariance matrix. When the sketch size exceeds the rank of the original covariance matrix, $\varepsilon_l=0$ and the bound collapses to $O(\sqrt{ldT})$, matching the guarantee of FedLinUCB.
\\
\textbf{Empirical evaluation.} We conduct experiments on both synthetic and real-world datasets, and the results show that, compared with the no-sketch baseline, the FSCLB algorithm sacrifices only a marginal amount of reward while significantly reducing both computational and communication costs.

\begin{table*}[!ht]
 \caption{Summary of the main results}
		\label{tab:compare}	
	\centering
	\resizebox{1\columnwidth}{!}{
	\centering
	\begin{threeparttable}
	\begin{tabular}{|ccccc|}
 \hline
\textbf{Previously work} &\textbf{Asynchronism} &\textbf{Regret} & \textbf{Per-round computation cost} & \textbf{Total communication cost}\\
  \hline    
       \textbf{Async-LinUCB} \cite{AsynchronousUpperConfidenceBound} & Partially &$\widetilde{O}(d\sqrt{T})$ & $O(d^3)$  & $O(d^3M^2\log T)$\\ 
       \textbf{FedLinUCB} \cite{FedLinUCB} & Fully&$\widetilde{O}(d\sqrt{T})$ & $O(d^3)$  & $O(d^3M^2\log T)$\\
    \hline
       \rowcolor{Gray}
   \textbf{FSCLB} (Ours) &Fully &	$\widetilde{O}((\sqrt{d}+\sqrt{M\varepsilon_l})\sqrt{lT})$ & $O(l^2d)$ & $O\left(ld^2M^2\log\left((1+\varepsilon_l)T\right)\right)$\\
	\hline
	\end{tabular}
	  \begin{tablenotes}[para, online,flushleft]
Regarding the comparison of asynchronism, “Partially” denotes partial asynchrony that still relies on certain synchronization strategies to complete updates, whereas “Fully” indicates complete asynchrony whose update and communication procedures require no synchronization whatsoever. $l$ is the sketch size with $l<d$. $\varepsilon_l$ is upper bounded by the spectral tail of the covariance
matrix. According to \cite{SOFUL}, if the rank of the selected-arm sequence smaller than the sketch size, then  $\varepsilon_l=0$ and our algorithm achieves a regret of order $\widetilde{O}(\sqrt{ldT})$. 
	\end{tablenotes}
			\end{threeparttable}
	}
\end{table*}

\section{Related Work}
\noindent\textbf{Federated bandits.} In large-scale bandit tasks, a larger number of agents must collaborate to achieve more accurate and faster learning. To mitigate potential privacy leakage in multi-agent collaborative learning, federated bandits have attracted growing attention \cite{LI2024110663,10.5555/3722577.3722714}. FL enables multiple agents to perform asynchronous collaborative learning: a central server aggregates the data from all agents, performs the necessary computations, and returns the updated results to every agents \cite{DBLP:journals/corr/KonecnyMYRSB16,Sahu2018FederatedOI,9311906,10.5555/3692070.3692621}. The optimization goal of FMAB is to minimize the sum of regrets across all agents \cite{6763073,10.5555/3666122.3669392}. Fed2-UCB \cite{Shi_Shen_2021} and UCB-TCOM \cite{DBLP:conf/iclr/WangYCLHTL23} are two representative homogeneous FMAB algorithms; their work establishes the classic bandit framework under federated learning. As a further extension, \cite{MutiagentMutiarm} combines combinatorial multi-armed bandits with federated learning and provides a general multi-agent framework. \cite{BitLevel} proposes a novel federated multi-armed bandit algorithm tailored to fully distributed networks; thid method simultaneously minimizes both individual and group regret by communicating with immediate neighbors. Our algorithm also operates under communication constraints, but in contrast to \cite{BitLevel}, we enforce the stricter regime where no pair of agents may exchange messages at all.
\\
\textbf{Federated contextual linear bandits.} In FMAB, the contextual linear bandit \cite{OFUL} is a more general and challenging direction. Here, multiple agents jointly solve the linear contextual bandit problem. \cite{pmlr-v48-korda16} proposed the DCB algorithm for P2P networks. \cite{Wang2020Distributed} considered both P2P and star communication patterns and achieved an $\widetilde{O}(d\sqrt{T})$ regret under synchronous communication. In the asynchronous setting, \cite{AsynchronousUpperConfidenceBound} gave the first asynchronous algorithm for star-shaped federated contextual linear bandits, attaining the same regret order. \cite{FedLinUCB} further generalizes asynchronous communication, enabling truly asynchronous federated contextual linear bandits and establishing an equally tight regret bound under standard assumptions. Extensions to cluster bandits \cite{pmlr-v180-liu22a,ourAAAI} broaden their practical scope.
\\
\textbf{Sketch bandits.} Sketching is a classic dimension-reduction technique \cite{10.1145/3055399.3055431,pmlr-v80-andoni18a,BeyondJohnson}. Introducing sketching techniques into linear bandits can effectively reduce the computational cost of matrix inversion \cite{Yu_Lyu_King_2017,SOFUL,wen2024matrixsketchingbanditscurrent}. \cite{FD} proposed the Frequent Directions sketch (FD), which enjoys strong theoretical guarantees in streaming settings and is well-suited for online scenarios. \cite{SOFUL} integrates FD with the classical OFUL algorithm in linear contextual bandits, yielding the SOFUL algorithm that reduces the per-round update cost. \cite{SCFD} proposes a variant of the FD method and applies it to linear bandits, achieving a lower regret bound compared to SOFUL. In distributed or federated learning, sketching is used to reduce the dimensional of matrix multiplications and aggregation operations, thereby enabling more communication-efficient transmission protocols \cite{10.5555/3618408.3619750,NEURIPS2024_0c433d14}. 

In this paper, we incorporate sketching techniques into fully asynchronous federated contextual linear bandits to reduce both computational and communication costs; a comparison with prior works is provided in Table \ref{tab:compare}.

\section{Preliminaries}
\subsection{Notation and Definition}
In our paper, every d-dimensional vector is represented as a column vector. For matrix $\bm{A}$ and $\bm{B}$, $[\bm{A};\bm{B}]$ denotes their column-wise concatenation. $[\bm{A}]_k$ denotes the matrix formed by taking the first $k$ rows of matrix $\bm{A}$. For matrix $\bm{A}\in \mathbb{R}^{d_1\times d_2}$, the computed SVD of $\bm{A}$ is defined as $\bm{A} = \bm{U}\bm{\Sigma}\bm{V}^{\top}$, where $\bm{U}$ is $d_1\times d_2$ matrix, $\bm{\Sigma}$ is $d_2\times d_2$ diagonal matrix, $\bm{V}$ is $d_2\times d_2$ matrix. We use $SVD(\bm{A})$ to denote the singular value decomposition of matrix $\bm{A}$ and obtain the corresponding $\bm{U}$, $\bm{V}^{\top}$ and $\bm{\Sigma}$. Let $\sigma_i(\bm{A})$ be the $i$-th largest singular value of $\bm{A}$. For theoretical bounds, the notation $\widetilde{O}(\cdot)$ omits logarithmic factors.
\subsection{Problem Setting}
\noindent\textbf{Federated contextual bandits.} At each round $t\in [T]$, an arbitrary agent $m_t \in [M]$ is active for participation. Then, this agent receives a decision set $D_t \subset \mathbb{R}^d$. The agent selects an appropriate action $\bm{x}_t\in D_t$ based on historical information and receives the corresponding reward $r_t$. We assume that the reward $r_t$ satisfies $r_t=\left \langle \bm{x}_t,\bm{\theta}^* \right \rangle+\eta_t$ for all $t\in [T]$, where $\eta_t$ is conditionally independent of $\bm{x}_t$ given $x_{1:t-1},m_{1:t-1},r_{1:t-1}$. According to \cite{FedLinUCB}, we make the following assumption:
\begin{assumption}
The noise $\eta_t$ is R-sub-Gaussian conditioning on $x_{1:t-1},m_{1:t-1},r_{1:t-1}$.
\[
\mathbb{E}[e^{\lambda \eta_t}|x_{1:t-1},m_{1:t-1},r_{1:t-1}]\leq exp(R^2\lambda^2/2).
\]
\end{assumption}
\begin{assumption}
$\left \|\bm{\theta}^*  \right \|_2\leq S$, $\left \|\bm{x}  \right \|_2\leq L$ for all action $\bm{x}_t\in D_t$, for all $t\in [T]$.
\end{assumption}
The goal of the agents is to collaboratively minimize the cumulative regret defined as
\begin{align}
Regret(T)\nonumber
&:=\sum_{t=1}^{T}\left ( \max_{\bm{x}\in D_t}\left \langle\bm{x},\bm{\theta}^*\right \rangle-\left \langle\bm{x}_t,\bm{\theta}^*\right \rangle \right ).
\end{align}
where $\bm{x}_t^*=argmax_{\bm{x}\in D_t}\left \langle\bm{x},\bm{\theta}^*\right \rangle$ is the optimal arm.

We use $\bm{X}_t=[\bm{x}_1,...,\bm{x}_t]^{\top}$ to define the matrix of all actions selected up to round $t$ by an arbitrary algorithm. 
\\
\textbf{Communication model.} We adopt the communication protocol with $M$ local user agents and one central server, where each user can communicate with the server by uploading and downloading data. Local users do not communicate with each other directly. Moreover, we adopt the \textit{asynchronous communication} paradigm: (1) there is no mandatory synchronization, (2) the communication between an agent and the server operates independently of other agents, without triggering additional communication. The \textit{switching cost} in online learning and bandits is defined as  the number of times the user updates its policy of selecting an action from the decision set, which equals to the number of communications in our model \cite{FedLinUCB}. In addition to the number of communication rounds, the amount of data transferred between user and server in each round also contributes to communication overhead. Here we define the \textit{communication volume} as how much data is exchanged between user and server during one communication round. For example, in FedLinUCB, when a communication round occurs the agent must upload a $d\times d$ local correlation matrix to the server and download a $d\times d$ matrix (the inverse of the regularized correlation matrix) plus a $1\times d$ vector of estimated values; thus, the per-round communication volume of FedLinUCB is $d^2+d=O(d^2)$. We define our communication cost $\text{Com}(T)$ as the combination of communication volume and switching cost.
\begin{align}
\nonumber
\text{Com}(T)
&=\text{switching cost} \times \text{communication volume}.
\end{align}
\textbf{Sketch.} Maintaining the full matrix $\bm{X}_T$ incurs substantial computational and communication overhead. Let $l<d$ is the sketch size. The goal of sketching is to keep an approximate matrix $\bm{S}_T\in \mathbb{R}^{l\times d}$ in place of $\bm{X}_T$ for policy updates, where $\bm{S}_T$ satisfies
\[
\bm{S}_T^{\top}\bm{S}_T \approx \bm{X}_T^{\top}\bm{X}_T.
\]

Frequent Directions (FD) \cite{FD} is a classical sketch algorithm that truncates the smallest singular value after performing an SVD of the covariance matrix at every update. In each round $t$, FD updates as follows:
\begin{align}
\nonumber &[\bm{U},\bm{\Sigma},\bm{V}^{\top}]{=}SVD\left([\bm{S}_{t-1}^{\top};\bm{x}_t]^{\top}\right),\\
\nonumber &\delta_t{=}\sigma_l^2(\bm{\Sigma}), \bm{S}_{t}{=}\sqrt{\bm{\Sigma}^2-\delta_t\bm{I}}\bm{V}^{\top}.
\end{align}

The sketch matrix serves as an approximate representation of the original covariance matrix, some information is lost during the sketching process (e.g., the truncation of singular values), resulting in a discrepancy between the sketch and the original covariance matrix. We quantify this discrepancy using the spectral error defined in \cite{SOFUL}:
\[
\varepsilon_l=\min_{k=0,...,l-1}\frac{\lambda_{d-k}+\lambda_{d-k+1}+...+\lambda_{d}}{\lambda(l-k)}.
\]

With this definition, $\varepsilon_l$ is upper bounded by the spectral tail of the covariance matrix under $l$ sketch size. According to the Proposition 1 in \cite{SOFUL} and Theorem 3.1 in \cite{FD}, we known that the approximation matrix is bounded by:
\[
\left \|\bm{S}_T^{\top}\bm{S}_T - \bm{X}_T^{\top}\bm{X}_T  \right \|_2 \leq \lambda \varepsilon_l.
\]
\section{Algorithm}
In this section, we present the  Federated Sketch Contextual Linear Bandits (FSCLB) algorithm and detail the key implementation techniques. For clarity, the fundamental notation used in the algorithms is summarized in Table \ref{tab: notation}.
\begin{table}
    \centering
    \begin{tabular}{ll}
        \hline
        Notation& Definition \\
        \hline
        $\bm{\widetilde{V}}_{m,t}$& the matrix that user $m$ downloads from \\
              &the server at round $t$ to update its policy\\

        $\bm{S}_{m,t}^{loc},\widetilde{\bm{S}}_t^{ser}$& sketch matrix in user $m$, server at time $t$ \\

        $\bm{S}_{m,t}$& sketch matrix correspond to $\bm{\widetilde{V}}_{m,t}$ \\

		$\rho_{m,t}^{loc}$& the sum of all singular values truncated during\\
&sketch updates at user $m$ up to round $t$ \\

		$\rho_t^{ser}$& the sum of all $\rho_{m,t}^{loc}$ uploaded by every user to\\
&the server up to round $t$ \\
        \hline
    \end{tabular}
    \caption{Notations}
    \label{tab: notation}
\end{table}
\begin{algorithm}[!ht] 
    \caption{Federated Sketch Contextual Linear Bandits} 
	\label{alg:FSCLB}
    \textbf{Input:} $\delta,\alpha {>}0$, sketch size $l<d$.\\ 
    \textbf{Parameter:} For all $m{\in} [M]$: $\bm{\widetilde{V}}_{m,1}{=}\bm{\widetilde{V}}_{0}^{ser}{\leftarrow}\lambda \bm{I}_{d\times d}$, $\bm{S}_{m,1}{=}\bm{S}_{m,0}^{loc}{=}\widetilde{\bm{S}}_0^{ser}{\leftarrow}\bm{0}_{l\times d}$, $\rho_{m,0}^{loc}{=}\widetilde{\rho}_{0}^{ser}{=}\rho_{0}^{ser}{\leftarrow}0$, $\bm{\hat{\theta}}_{m,1}{\leftarrow}\bm{0}$, $\Delta_{m,1}{=}\Delta_{0}^{ser}{\leftarrow}0$, $\bm{b}_{0}^{ser}{=}\bm{b}_{m,0}^{loc}{\leftarrow}\bm{0}$, $\bm{\widetilde{\Sigma}}_{0}^{ser}{\leftarrow}\bm{0}$. \\
    \textbf{Output:} 
    \begin{algorithmic}[1]
    \FOR{round $t=1,...,T$}
	\STATE Activate $m_t$ and receive $D_t$
	\STATE $\bm{x}_t{\leftarrow} argmax_{\bm{x}\in D_t}\left\{\left \langle \bm{\hat{\theta}}_{m_t,t},\bm{x}\right \rangle{+}\beta_{m_t,t}(\delta)||\bm{x}||_{\bm{\widetilde{V}}_{m_t,t}^{-1}}\right \}$
	\STATE $m_t$ play $\bm{x}_t$ and receive $r_t$
	\STATE $(\bm{X}_{m_t,t}^{loc})^{\top}\bm{X}_{m_t,t}^{loc}{\leftarrow}(\bm{X}_{m_t,t-1}^{loc})^{\top}\bm{X}_{m_t,t-1}^{loc}{+}\bm{x}_t\bm{x}_t^{\top}$
	\STATE $\bm{S}_{m_t,t}^{loc}{\leftarrow}\left[(\bm{S}_{m_t,t-1}^{loc})^{\top};\bm{x}_t\right]^{\top}$, $\bm{b}_{m_t,t}^{loc}{\leftarrow}\bm{b}_{m_t,t-1}^{loc}{+}r_t\bm{x}_t$
	\STATE $[\bm{U}_1,\bm{\Sigma}_1,\bm{V}_1^{\top}] {\leftarrow} SVD(\bm{S}_{m_t,t}^{loc})$
	\STATE $\delta_{m_t,t}^{FL}{\leftarrow}\sigma_l^2(\bm{\Sigma}_1)$, $\rho_{m_t,t}^{loc}{\leftarrow}\rho_{m_t,t-1}^{loc}{+}\delta_{m_t,t}^{FL}$
	\STATE $\hat{\bm{\Sigma}}_1{\leftarrow}\sqrt{\max\{\bm{\Sigma}_1^2{-}\delta_{m_t,t}^{FL}\bm{I},0\}}$
    \STATE $\bm{S}_{m_t,t}^{loc}{\leftarrow}[\hat{\bm{\Sigma}}_1]_{l\times d}\bm{V}_1^{\top}\in \mathbb{R}^{l\times d}$ 
	\IF{$l< 0.4d$}
	\STATE $[\bm{U}_2,\bm{\Sigma}_2,\bm{V}_2^{\top}] {\leftarrow} SVD\left(\left[(\bm{S}_{m_t,t})^{\top};(\bm{S}_{m_t,t}^{loc})^{\top}\right]^{\top}\right)$
	\STATE $e_{t}^{(1)}{\leftarrow}(\lambda {+}\rho_{m_t,t}^{loc}{+}\Delta_{m_t,t})^{d-2l}$
    \STATE $e_{t}^{(2)}{\leftarrow}\prod_{i=1}^{2l}\left (\sigma_{i}^2(\bm{\Sigma}_2){+}\lambda{+}\rho_{m_t,t}^{loc}{+}\Delta_{m_t,t}\right)$
    \STATE $det_{m_t,t}^{(1)} {\leftarrow}e_{t}^{(1)}\cdot e_{t}^{(2)}$
	\ELSE
	\STATE $det_{m_t,t}^{(1)}{\leftarrow}det\left(\bm{\widetilde{V}}_{m_t,t}{+}(\bm{S}_{m_t,t}^{loc})^{\top}\bm{S}_{m_t,t}^{loc}{+}\rho_{m_t,t}^{loc}\bm{I}\right)$ 
	\ENDIF
	\IF{$det_{m_t,t}^{(1)}{>}(1{+}\alpha)det(\bm{\widetilde{V}}_{m_t,t})$}
	\STATE $m_t$ upload $\bm{S}_{m_t,t}^{loc},\rho_{m_t,t}^{loc}$ to server // upload phase
	\STATE Run Algorithm \ref{alg:communication}, and $m_t$ download the output
	\STATE $m_t$ compute \\$\bm{\widetilde{V}}_{m_t,t+1}^{-1}{\leftarrow}\frac{1}{\lambda{+}\Delta_{m_t,t+1}}\left (\bm{I}{-}\bm{S}_{m_t,t+1}^{\top}\bm{H}_{m_t,t+1}\bm{S}_{m_t,t+1} \right )$
	\STATE $m_t$ compute $\beta_{m_t,t+1}(\delta)$ as equation \ref{eq:confidence bound}
	\ELSE
	\STATE $\bm{\widetilde{V}}_{m_t,t+1}{\leftarrow}\bm{\widetilde{V}}_{m_t,t}$, $\bm{\widetilde{V}}_{t}^{ser}{\leftarrow}\bm{\widetilde{V}}_{t-1}^{ser}$, $\bm{\hat{\theta}}_{m_t,t+1}{\leftarrow}\bm{\hat{\theta}}_{m_t,t}$, \\$\bm{S}_{m_t,t+1}{\leftarrow}\bm{S}_{m_t,t}$, $\widetilde{\bm{S}}_{t}^{ser}{\leftarrow}\widetilde{\bm{S}}_{t-1}^{ser}$, $\bm{b}_{t}^{ser}{\leftarrow}\bm{b}_{t-1}^{ser}$,
    \\$\Delta_{m_t,t+1}{\leftarrow}\Delta_{m_t,t}$, $\Delta_t^{ser}{\leftarrow}\Delta_{t-1}^{ser}$, $\beta_{m_t,t+1}{\leftarrow}\beta_{m_t,t}$,
	\\$\rho_t^{ser}{\leftarrow}\rho_{t-1}^{ser}$, $\widetilde{\rho}_t^{ser}{\leftarrow}\widetilde{\rho}_{t-1}^{ser}$, 
	\ENDIF
	\FOR{all user $m\ne m_t$}
	\STATE $\bm{S}_{m,t}^{loc}{\leftarrow}\bm{S}_{m,t-1}^{loc}$, $\bm{b}_{m,t}^{loc}{\leftarrow}\bm{b}_{m,t-1}^{loc}$, $\rho_{m,t}^{loc}{\leftarrow}\rho_{m,t-1}^{loc}$, 
    \\$\bm{\widetilde{V}}_{m,t+1}^{-1}{\leftarrow}\bm{\widetilde{V}}_{m,t}^{-1}$, $\bm{S}_{m,t+1}{\leftarrow}\bm{S}_{m,t}$, $\bm{\hat{\theta}}_{m,t+1}{\leftarrow}\bm{\hat{\theta}}_{m,t}$, \\$\beta_{m,t+1}{\leftarrow}\beta_{m,t}$, $\Delta_{m,t+1}{\leftarrow}\Delta_{m,t}$, $\bm{H}_{m,t+1}{\leftarrow}\bm{H}_{m,t}$
	\ENDFOR
    \ENDFOR 
    \end{algorithmic}
    \end{algorithm}
\begin{algorithm}[!ht] 
    \caption{Operations Performed on Server} 
	\label{alg:communication}
    \textbf{Input}: $m_t$ upload $\bm{S}_{m_t,t}^{loc},\rho_{m_t,t}^{loc}$ // upload phase\\
    \textbf{Output}:
    \begin{algorithmic}[1] 
	\STATE $\rho_t^{ser}\leftarrow\rho_{t-1}^{ser}+\rho_{m_t,t}^{loc}$
	\STATE $(\bm{X}_{t}^{ser})^{\top}\bm{X}_{t}^{ser}\leftarrow(\bm{X}_{t-1}^{ser})^{\top}\bm{X}_{t-1}^{ser}+(\bm{X}_{m_t,t}^{loc})^{\top}\bm{X}_{m_t,t}^{loc}$ // This line is only for proof
	\STATE compute SVD: \\$[\bm{U},\bm{\Sigma},\bm{V}^{\top}]\leftarrow SVD\left(\left[(\widetilde{\bm{S}}_{t-1}^{ser})^{\top};(\bm{S}_{m_t,t}^{loc})^{\top}\right]^{\top}\right)$
	\STATE $\delta_t^{ser}\leftarrow\sigma_l^2(\bm{\Sigma})$, $\widetilde{\rho}_t^{ser}\leftarrow\widetilde{\rho}_{t-1}^{ser}+\delta_t^{ser}$
	\STATE $\hat{\bm{\Sigma}}\leftarrow\sqrt{\max\{\bm{\Sigma}^2-\delta_t^{ser}\bm{I},0\}}$
	\STATE $\widetilde{\bm{S}}_t^{ser}\leftarrow[\hat{\bm{\Sigma}}]_{l\times d}\bm{V}^{\top}\in \mathbb{R}^{l\times d}$
	\STATE $\Delta_t^{ser}\leftarrow\widetilde{\rho}_t^{ser}+\rho_t^{ser}$, $\bm{b}_{t}^{ser}\leftarrow\bm{b}_{t-1}^{ser}+\bm{b}_{m_t,t}^{loc}$
	\STATE $\bm{H}^{ser}_t\leftarrow(\hat{\bm{\Sigma}}^2+(\Delta_t^{ser}+\lambda)\bm{I})^{-1}$, $\bm{H}_{m_t,t+1}\leftarrow\bm{H}^{ser}_t$
	\STATE $\bm{\widetilde{V}}_{t}^{ser}\leftarrow(\lambda+\Delta_t^{ser})\bm{I}+(\widetilde{\bm{S}}_t^{ser})^{\top}\widetilde{\bm{S}}_t^{ser}$, $\bm{\widetilde{V}}_{m_t,t+1}\leftarrow\bm{\widetilde{V}}_{t}^{ser}$
	\STATE $\bm{S}_{m_t,t+1}\leftarrow\widetilde{\bm{S}}_t^{ser}$, $\bm{\hat{\theta}}_{m_t,t+1}\leftarrow(\bm{\widetilde{V}}_{t}^{ser})^{-1}\bm{b}_{t}^{ser}$
	\STATE $\Delta_{m_t,t+1}\leftarrow\Delta_t^{ser}$
	\STATE $\bm{S}_{m_t,t}^{loc}\leftarrow\bm{0}$, $\bm{b}_{m_t,t}^{loc}\leftarrow\bm{0}$, $\rho_{m_t,t}^{loc}\leftarrow 0$
	\STATE Compute and return $\bm{S}_{m_t,t+1},\bm{\hat{\theta}}_{m_t,t+1},det(\bm{\widetilde{V}}_{m_t,t+1})$, $\bm{H}_{m_t,t+1}, \Delta_{m_t,t+1}$ to $m_t$ // download phase
    \end{algorithmic}
    \end{algorithm}
\subsection{Determinant Calculation and Data Upload}
In FedLinUCB \cite{FedLinUCB}, the local user need to compute the determinant in each round to decide whether communication is required; traditional determinant computation incurs $O(d^3)$ complexity. In FSCLB, we compute the determinant of the corresponding matrix via its singular values. Note that the singular values of this covariance matrix are exactly the squares of the singular values of the matrix $\bm{X}_t$ formed by the sequence of arms selected by the algorithm. Since the determinant of a matrix equals the product of all its singular values, we can therefore compute the required determinant from the singular values of $\bm{X}_t$. However, directly computing these singular values is prohibitively expensive, so we employ a sketching method to obtain a matrix $\bm{S}_t$ whose singular values are easy to compute and which approximates $\bm{X}_t$; the determinant is then obtained by squaring these singular values, adding $\lambda$, and taking the product.

In FSCLB, each user needs two determinants: 1) the determinant of the matrix formed by summing the $\widetilde{\bm{V}}_{m,t}$ downloaded from the server and the user’s locally accumulated covariance matrix, and 2) the determinant of $\widetilde{\bm{V}}_{m,t}$. Since $\widetilde{\bm{V}}_{m,t}$ remains unchanged when no communication occurs, its determinant can be obtained from the server during the last communication round and then stored locally; hence no need to compute it. For the determinant in item 1), each user keeps a sketch matrix $(\bm{S}_{m,t}^{loc})^{\top}\bm{S}_{m,t}^{loc}$ that approximates the local covariance matrix; so the user only needs to maintain the $\bm{S}_{m,t}$ downloaded from the server. Moreover, the singular values of $\bm{S}_{m,t}^{\top}\bm{S}_{m,t}+(\bm{S}_{m,t}^{loc})^{\top}\bm{S}_{m,t}^{loc}$ are the squares of the singular values of the concatenated matrix $\left[\bm{S}_{m,t}^{\top};(\bm{S}_{m,t}^{loc})^{\top}\right]^{\top}$, which has size at most $2l\times d$.  
The SVD of $\left[\bm{S}_{m,t}^{\top};(\bm{S}_{m,t}^{loc})^{\top}\right]^{\top}$ costs $O(l^2d)$; squaring the resulting singular values and adding the regularization terms yields the determinant by a sequential product. When $l\ll d$ this is cheaper than the $O(d^3)$ cost of computing the determinant directly. We implement this procedure in Lines 12–15 of Alg \ref{alg:FSCLB}. We use $det(\bm{A})$ to denote the determinant of matrix $\bm{A}$, and we use $det_{m_t,t}^{(1)}$ as a shorthand for the determinant $det\left(\bm{\widetilde{V}}_{m_t,t} {+}(\bm{S}_{m_t,t}^{loc})^{\top}\bm{S}_{m_t,t}^{loc} {+} \rho_{m_t,t}^{loc}\bm{I}\right)$.
\begin{remark}
Although computing the determinant via singular values is asymptotically cheaper than the traditional determinant algorithm, the required sketch updates introduce extra cost. Our algorithm delivers clear gains when $l\ll d$; if the sketch size is close to the original dimension, the computational savings become marginal. We therefore insert conditional checks in the algorithm (line 11 in Algorithm \ref{alg:FSCLB}) to ensure that the cumulative cost of the extra SVD computations never exceeds the cost of the direct determinant calculation. Analysis shows that, when $l<0.4d$, the combined time complexity of all these additional steps is provably lower than that of computing the determinant directly.
\end{remark}

\subsection{SCFD}
We define the local covariance matrix $(\bm{X}_{m,t}^{loc})^{\top}\bm{X}_{m,t}^{loc}=\sum_{i=1,m_i=m}^{t}\bm{x}_i\bm{x}_i^{\top}$ and maintain its sketch via SCFD \cite{SCFD}. Unlike conventional FD, SCFD returns the truncated singular values to the sketch, yielding a more accurate estimate. In FSCLB, we maintain $\rho^{loc}_{m,t}$ to record the sum of singular values truncated in local; after each sketch update, we add $\rho^{loc}_{m,t}\bm{I}$ back into the sketch matrix $(\bm{S}_{m,t}^{loc})^{\top}\bm{S}_{m,t}^{loc}$ to approximate the local covariance matrix, and we use this sketch matrix to evaluate the determinant criterion. Specifically, SCFD updates produce the approximation:
\[
(\bm{S}_{m,t}^{loc})^{\top}\bm{S}_{m,t}^{loc}+\rho^{loc}_{m,t}\bm{I}\approx (\bm{X}_{m,t}^{loc})^{\top}\bm{X}_{m,t}^{loc}.
\]

We adopt SCFD instead of classical FD not only to tighten regret but also to keep the communication trigger valid. In FedLinUCB \cite{FedLinUCB}, an upload is triggered whenever the determinant of the cumulative covariance matrix has grown sufficiently. This demands a monotonically non-decreasing determinant, whereas the sketch produced by classical FD updates not only fails to guarantee monotonicity but may also truncate the local growth. Property 1 in \cite{SCFD} shows that SCFD updates not only satisfies the monotonicity requirement but also preserves the local growth information, satisfying the asynchronous communication rule and yielding the theoretical guarantee we need.

\subsection{Double Sketch on Server}
Standard SCFD updates the sketch matrix by concatenating a single vector; in contrast, on the server side of our model the sketches from different users are merged, so the server performs the sketch update via matrix concatenation—an operation that differs from the single-agent case.

We extend the SCFD to a multi-sketch merging update: on the server, after concatenating the incoming sketch matrices we apply a second sketching step, effectively performing a double-sketch (line 3-6 in Algorithm \ref{alg:communication}). This guarantees that the matrix returned to the users remains $l\times d$. A plain concatenation could increase the rank and destroy the desired computational and communication savings; the second sketch keeps the merged matrix within the required sketch size. 

After the double-sketch merge on the server, additional singular values are truncated. Following SCFD, we accumulate these extra truncated values into a scalar $\widetilde{\rho}_t^{ser}$ and add $\widetilde{\rho}_t^{ser}\bm{I}$ back to the merged sketch matrix in the same way. The singular values that must be restored to the sketch matrix on the server consist of two parts: 1) $\rho_t^{ser}$, the sum of all local truncation values uploaded by the users, and 2) $\widetilde{\rho}_t^{ser}$, the truncation values accumulated during the second sketch applied to the merged matrix. The server forms $\Delta_t^{ser}=\rho_t^{ser}+\widetilde{\rho}_t^{ser}$ and use $\Delta_t^{ser}$ to approximates the original covariance matrix:
\[
(\widetilde{\bm{S}}_t^{ser})^{\top}\widetilde{\bm{S}}_t^{ser}+\Delta_t^{ser}\bm{I} \approx (\bm{X}_{t}^{ser})^{\top}\bm{X}_{t}^{ser}.
\]
where $(\bm{X}_{t}^{ser})^{\top}\bm{X}_{t}^{ser}$ denotes the aggregate of all users' covariance matrices without sketching.

By performing sketch updates via SCFD and combining them with the sketch-merging procedure proposed above, we can construct the confidence interval for each user as:
\begin{align}
\label{eq:confidence bound}
\nonumber
\beta_{m,t}(\delta)&{=}\widetilde{M}\left ( R\sqrt{d\log{\left (\frac{1+TL^2/(\widetilde{\alpha}\lambda)}{\delta}\right )}}+\sqrt{\lambda}S \right )\\
&+\left(\sqrt{\lambda}+\sqrt{\frac{\Delta_{m,t}}{\lambda}}\right)S.
\end{align}
where $\widetilde{M}=\sqrt{1+M\alpha}+M\sqrt{2\alpha}$.
\begin{lemma}
With probability at least $1-\delta$, for any $0< t\leq T$ and $m\in [M]$, we have
\[
\left \|\hat{\bm{\theta}}_{m,t}-\bm{\theta}^*  \right \|_{\bm{\widetilde{V}}_{m,t}}\leq \beta_{m,t}(\delta).
\]
\end{lemma}

\subsection{Computation and Communication}
The computational cost of FSCLB splits into local and server parts. As Section 5 shows, the total number of communication rounds is logarithmic, and the server only computes when communication occurs. Therefore, the per-round cost is dominated by local computation. In local user, the dominant computational cost comes from the SVD. As shown in \cite{FD}, SVD costs $O(l^2d)$ in each round. 

Regarding communication cost, both the determinant and the matrix used to compute the estimate can be reconstructed from their corresponding $l\times d$ sketch matrices. Moreover, following \cite{SCFD}, $\widetilde{\bm{V}}_{m,t}^{-1}$ can be obtained from the diagonal matrix $\bm{H}_{m,t}$ generated during the server-side sketch update together with the sketch $\bm{S}_{m,t}$ (line 22 in Alg \ref{alg:FSCLB}). To compute $\beta_{m,t}(\delta)$, the local user only needs to know the value of $\Delta_{m_t,t}$. Because the matrix $\bm{H}$ obtained on the server is diagonal, only its $l$ entries are sent; sending the sketch itself requires $l\times d$ elements. Hence, local user uploads and downloads $O(ld)$ communication volume, which is smaller than the previous $O(d^2)$ communication volume.

\begin{remark}
Note that in the communication-check step of FSCLB (line 19 in Algorithm \ref{alg:FSCLB}), $det_{m_t,t}^{(1)}$ must be compared with $det(\bm{\widetilde{V}}_{m_t,t})$; in fact, $det(\bm{\widetilde{V}}_{m_t,t})$ does not need to be computed locally. During the download phase (line 13 in Algorithm \ref{alg:communication}), the server can send the pre-computed $det(\bm{\widetilde{V}}_{m_t,t})$ to the local agent, and since the matrix $\bm{\widetilde{V}}_{m_t,t}$ remains unchanged until the next communication round, $det(\bm{\widetilde{V}}_{m_t,t})$ stays constant and do not need to be re-calculated.
\end{remark}

\section{Regret Analysis}
\begin{theorem}
With probability
at least $1-\delta$, the expected cumulative regret of Algoritm \ref{alg:FSCLB} is upper bounded by
\begin{align}
Regret(T)
\nonumber
&\leq 2dSLM\ln(1+TL^2/\lambda)\\
&+2\beta\sqrt{2(1+M\alpha)}\cdot \widetilde{f}(l,T).
\end{align}
where
\[
\widetilde{f}(l,T)=\sqrt{2lT\ln\left (1+\frac{TL^2}{l\lambda} \right )+2dT\ln\left (1+\varepsilon_l \right )},
\]
\[
\beta=\widetilde{O}\left((1+M\sqrt{\alpha}+\sqrt{(1/\alpha+M)\varepsilon_l})\sqrt{d}\right).
\]

Additionally, the total communication times is bounded by 
\[
2d(M+1/\alpha)\log\left( (1+\varepsilon_l)\left(1+\frac{ TL^2}{\lambda d}\right)\right).
\]

The total computation cost is bounded by 
\[
O\left(\min\{2l,d\}^2\left(dT+(M+1/\alpha)d^2\log\left((1+\varepsilon_l)T\right)\right)\right).
\]

The total communication cost is bounded by 
\[
O\left(ld^2(M+1/\alpha)\log\left((1+\varepsilon_l)T\right)\right).
\]
\end{theorem}
For the upper bound of $\beta_{m,t}$, we can show that
\begin{lemma}
For all $m\in[M]$ and all $t\in [T]$, it holds that
\[
\beta_{m,t}(\delta)\leq \beta.
\]
\end{lemma}

According to \cite{SOFUL}, if we set $l\geq rank(\bm{X}_T^{\top}\bm{X}_T)$, then $\varepsilon_l=0$. Furthermore, if we set $\alpha=1/M^2$ additionally, we can get $Regret(T)=\widetilde{O}(\sqrt{ldT})$. The detailed proof is provided in the Appendix.

\section{Experiment}
We compare our algorithm against two baseline: 1) FedLinUCB \cite{FedLinUCB}, 2) Random, in each round, every active user selects arm randomly. We ran 20 independent trials of 20 000 rounds.

\subsection{Simulation Dataset}
We first validate our conclusions on simulated data. In simulation datasets, we set the dimension to $d = 50$ and $100$, and the corresponding sketch size $l$ to 20 and 40. Each round contains 10 arms, $M = 10$ users, and we set $\alpha = 1$. We compare the algorithms in terms of total reward, total computation cost, and overall communication cost $Com(T)$. 1) Reward: the cumulative difference between the reward of the arm selected in each round and the reward that would have been obtained by the optimal arm. 2) Computation: the running time (in seconds) required to run the algorithm. 3) Communication cost $Com(T)$ (as defined in Section 3.2): the cumulative number of scalar elements uploaded and downloaded across all communication rounds (e.g., entries of matrices and vectors). The results are reported in Figures \ref{fig:reward},\ref{fig:time} and \ref{fig:communication}.

Experiments confirm the effectiveness of FSCLB. Compared with FedLinUCB, the total reward of FSCLB is slightly lower; this is because the SCFD module approximates the original covariance matrix with a sketched one, inevitably incurring some information loss. The benefit is a marked reduction in computation and communication costs. As can be seen, both running time and communication complexity are clearly reduced relative to the conventional approach, while the accompanying drop in reward is negligible.

\begin{figure}[!ht]
\centering
\begin{subfigure}{.23\textwidth}
  \centering
  \includegraphics[width=\linewidth,height=1\textwidth]{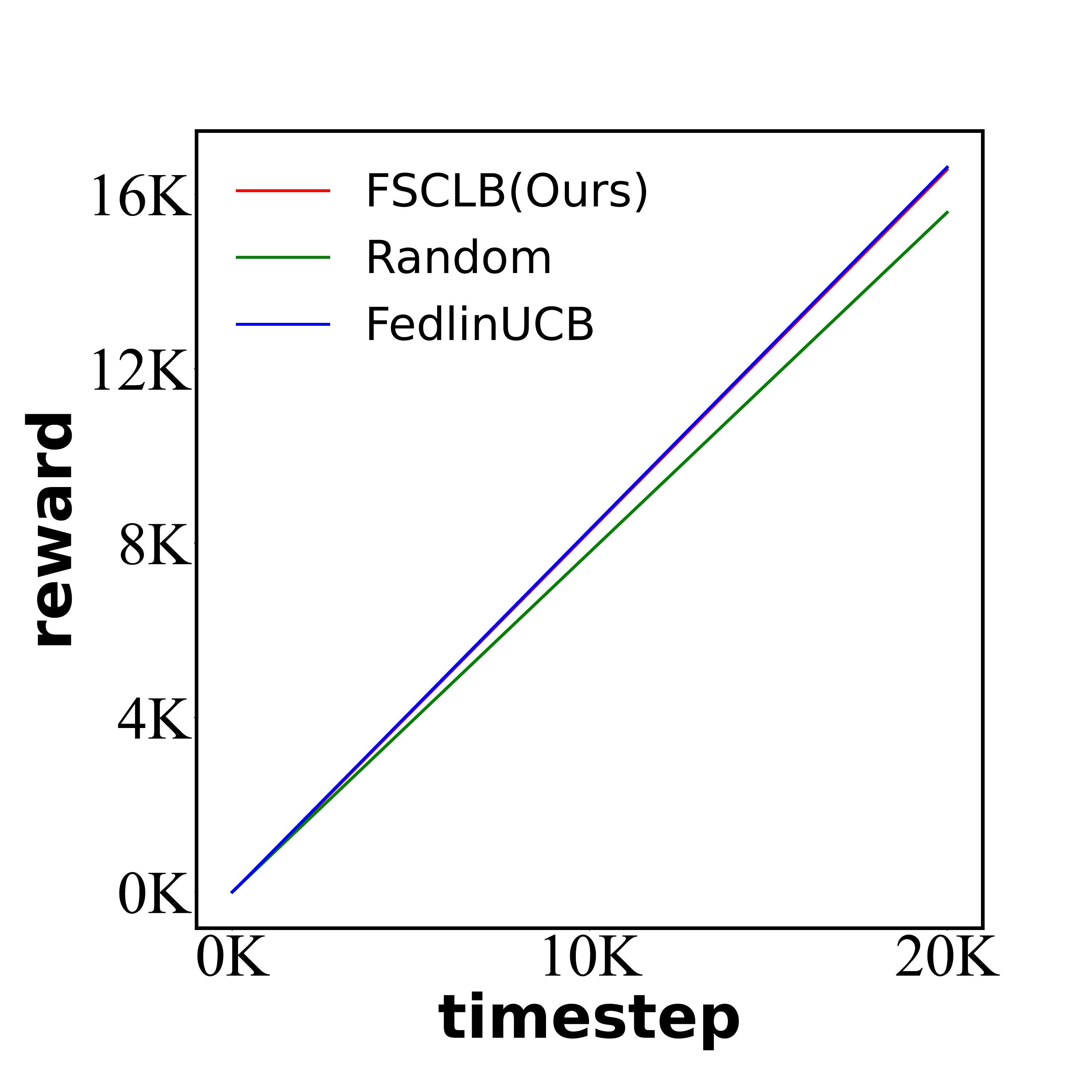}  
  \caption{50 dimension}
\end{subfigure}
\begin{subfigure}{.23\textwidth}
  \centering
  \includegraphics[width=\linewidth,height=1\textwidth]{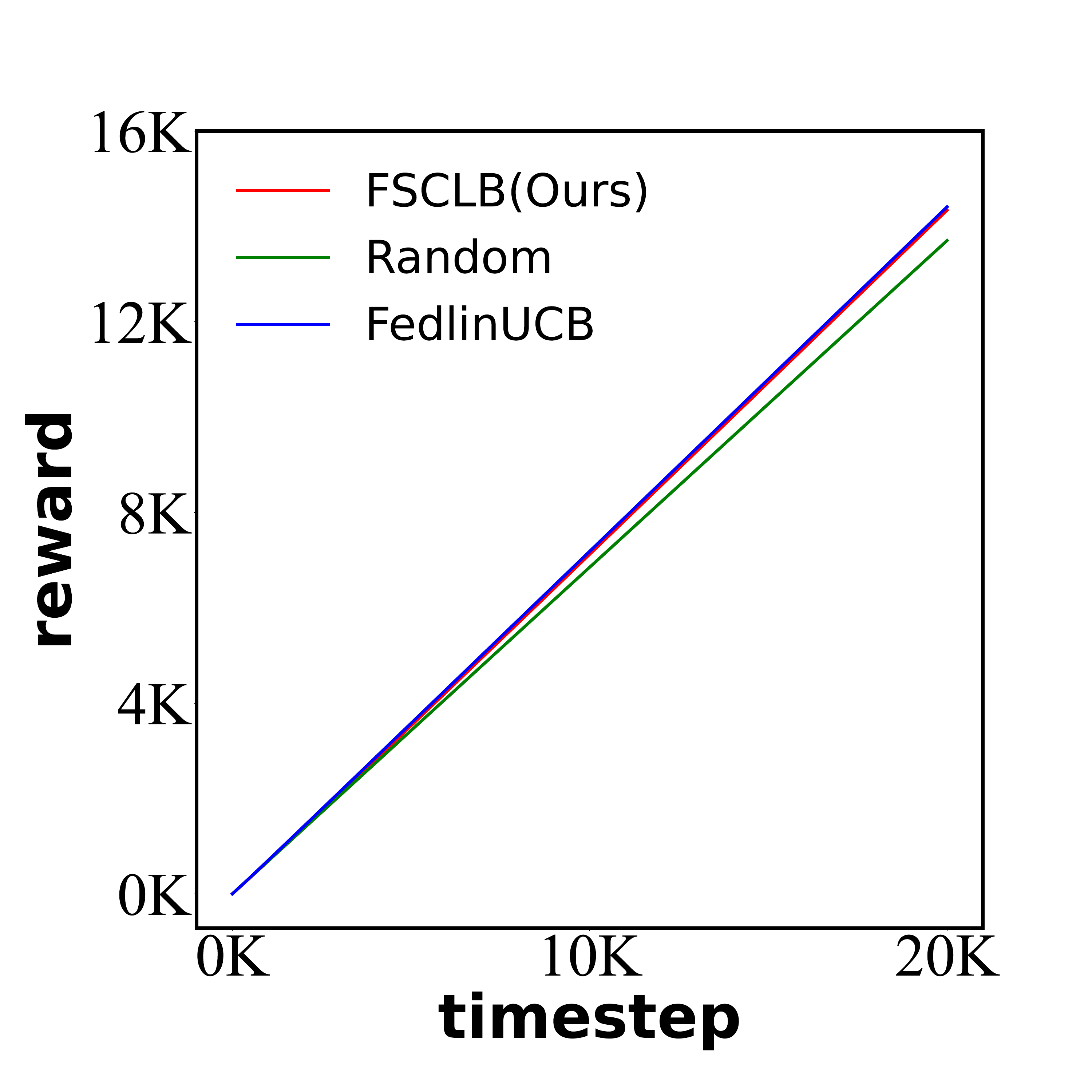}  
  \caption{100 dimension}
\end{subfigure}
\caption{Reward in simulation dataset}
\label{fig:reward}
\end{figure}
\begin{figure}[!ht]
\centering
\begin{subfigure}{.23\textwidth}
  \centering
  \includegraphics[width=\linewidth,height=1\textwidth]{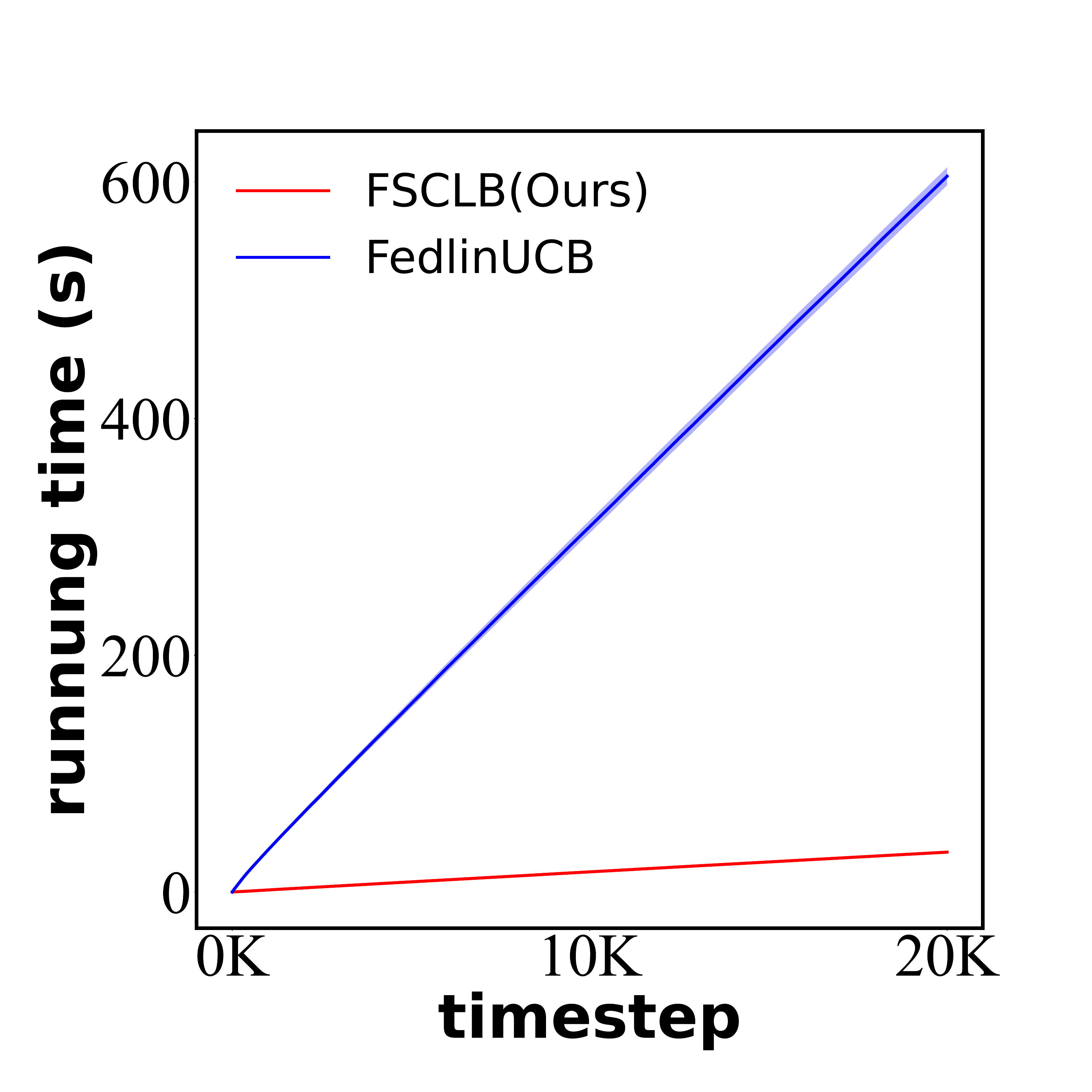}  
  \caption{50 dimension}
  \end{subfigure}
\begin{subfigure}{.23\textwidth}
  \flushleft
\includegraphics[width=\linewidth,height=1\textwidth]{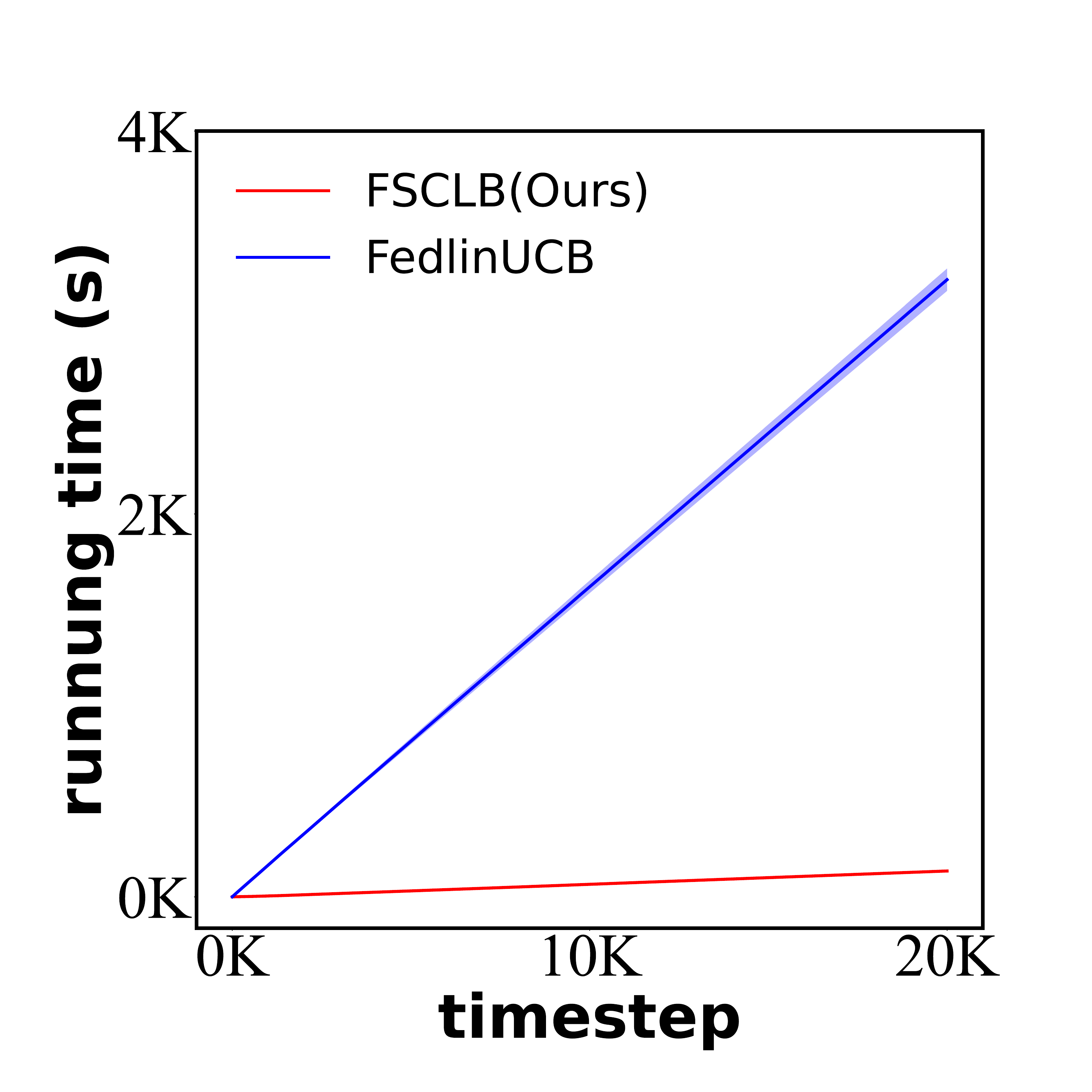}  
  \caption{100 dimension}
\end{subfigure}
\caption{Computation cost in simulation dataset}
\label{fig:time}
\end{figure}
\begin{figure}[!ht]
\centering
\begin{subfigure}{.23\textwidth}
  \centering
  \includegraphics[width=\linewidth,height=1\textwidth]{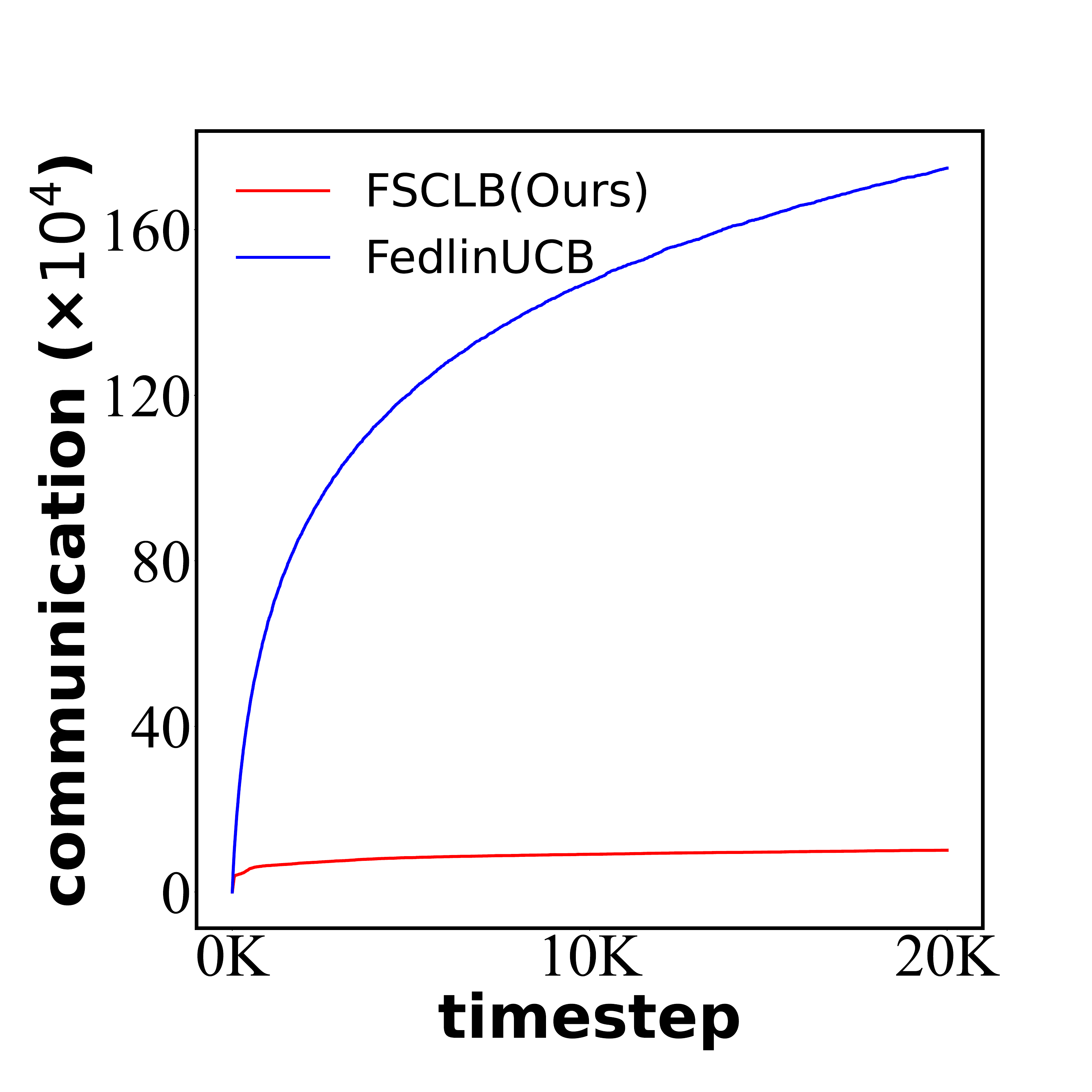}  
  \caption{50 dimension}
\end{subfigure}
\begin{subfigure}{.23\textwidth}
  \centering
  \includegraphics[width=\linewidth,height=1\textwidth]{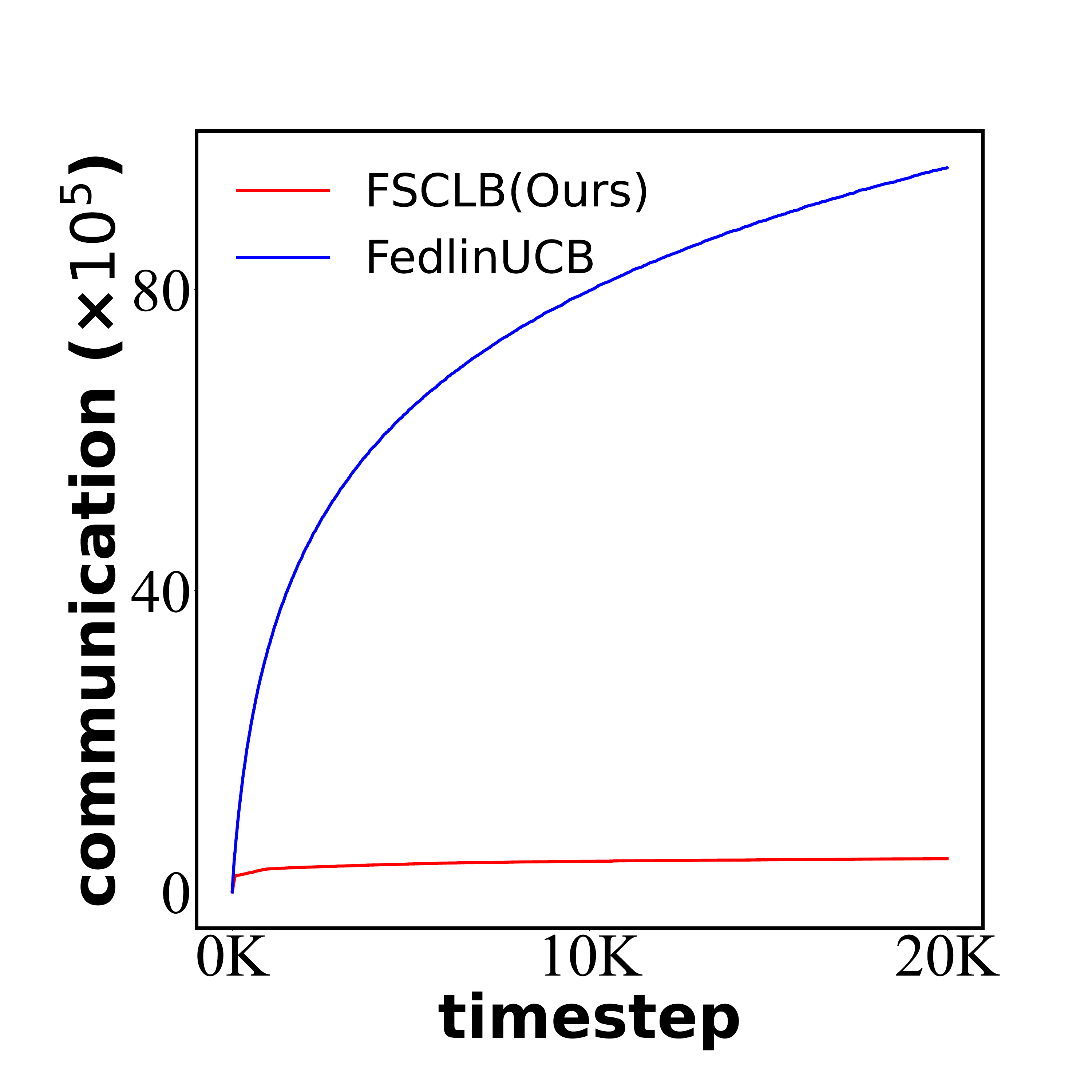}  
  \caption{100 dimension}
\end{subfigure}
\caption{Communication in simulation dataset}
\label{fig:communication}
\end{figure}

\subsection{Real Dataset}
We selected two high-dimensional public datasets from the OpenML repository: SatImage with $d=37$ (ID 23) and MFeat with $d=48$ (ID 182). Our feedback setting follows \cite{SOFUL}, and set $l=12,20$ respectively; the experimental results are shown in Fig \ref{fig:real_reward},\ref{fig:real_time} and \ref{fig:real_communication}.

On real-world setting, rewards are binary: 1 if the optimal arm is pulled and 0 otherwise. Under this feedback, the Random strategy deteriorates sharply, whereas FSCLB still matches the performance of FedLinUCB. These results show that, even in the setting where arm-selection is highly sensitive, FSCLB can maintain strong reward performance with lower computation and communication costs.

\begin{figure}[!ht]
\centering
\begin{subfigure}{.23\textwidth}
  \centering
  \includegraphics[width=\linewidth,height=1\textwidth]{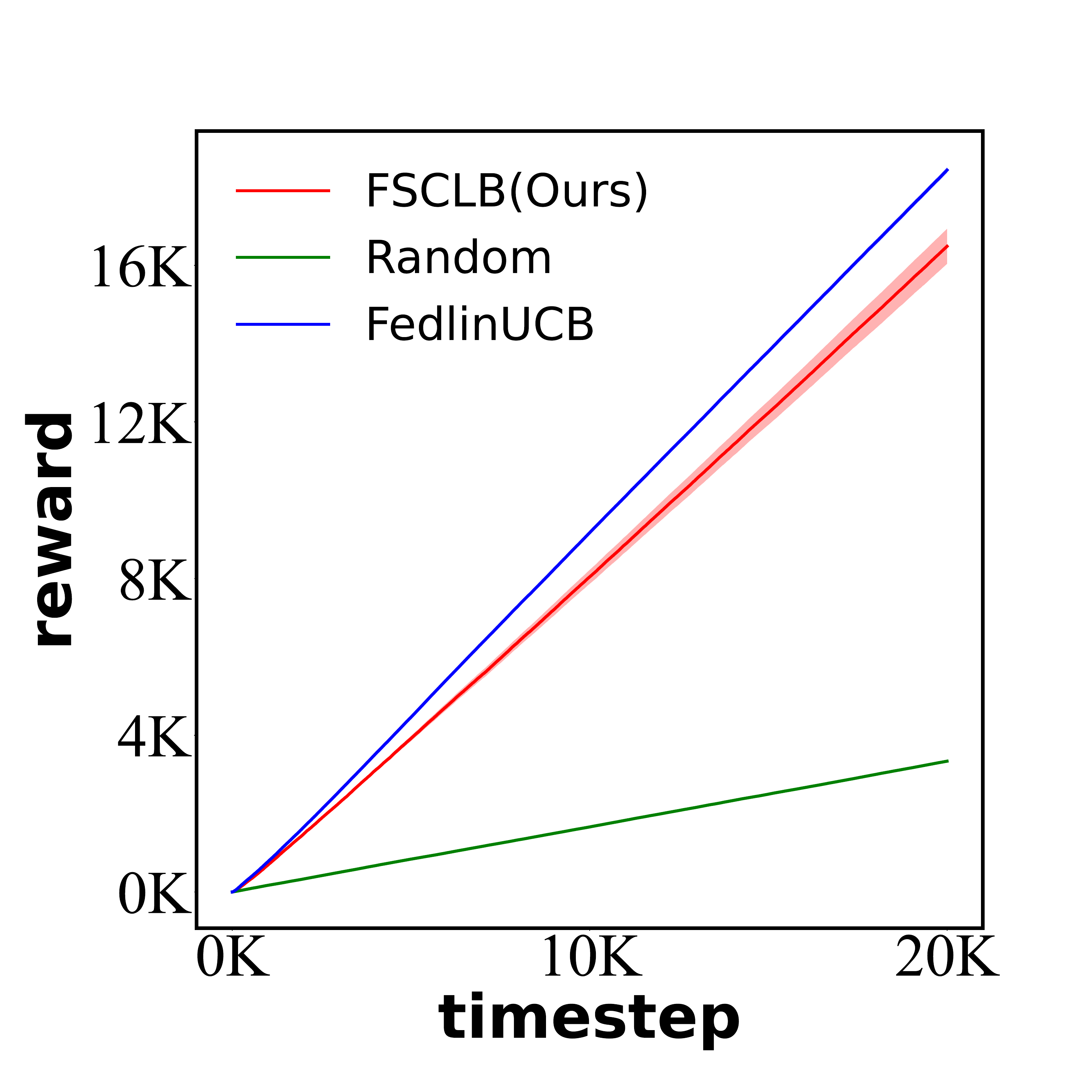}  
  \caption{satimage}
\end{subfigure}
\begin{subfigure}{.23\textwidth}
  \centering
  \includegraphics[width=\linewidth,height=1\textwidth]{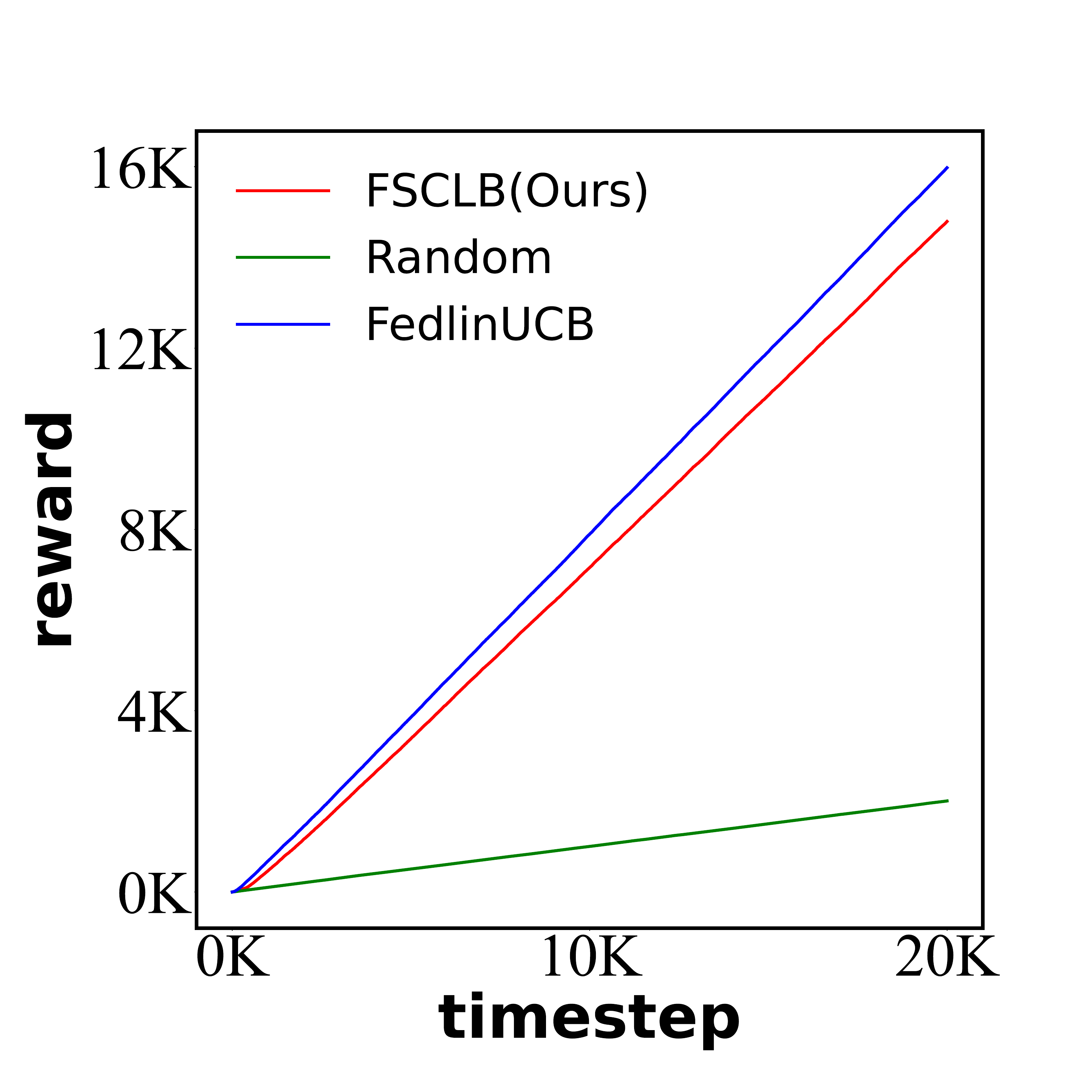}  
  \caption{mfeat}
\end{subfigure}
\caption{Reward in real dataset}
\label{fig:real_reward}
\end{figure}
\begin{figure}[!ht]
\centering
\begin{subfigure}{.23\textwidth}
  \centering
  \includegraphics[width=\linewidth,height=1\textwidth]{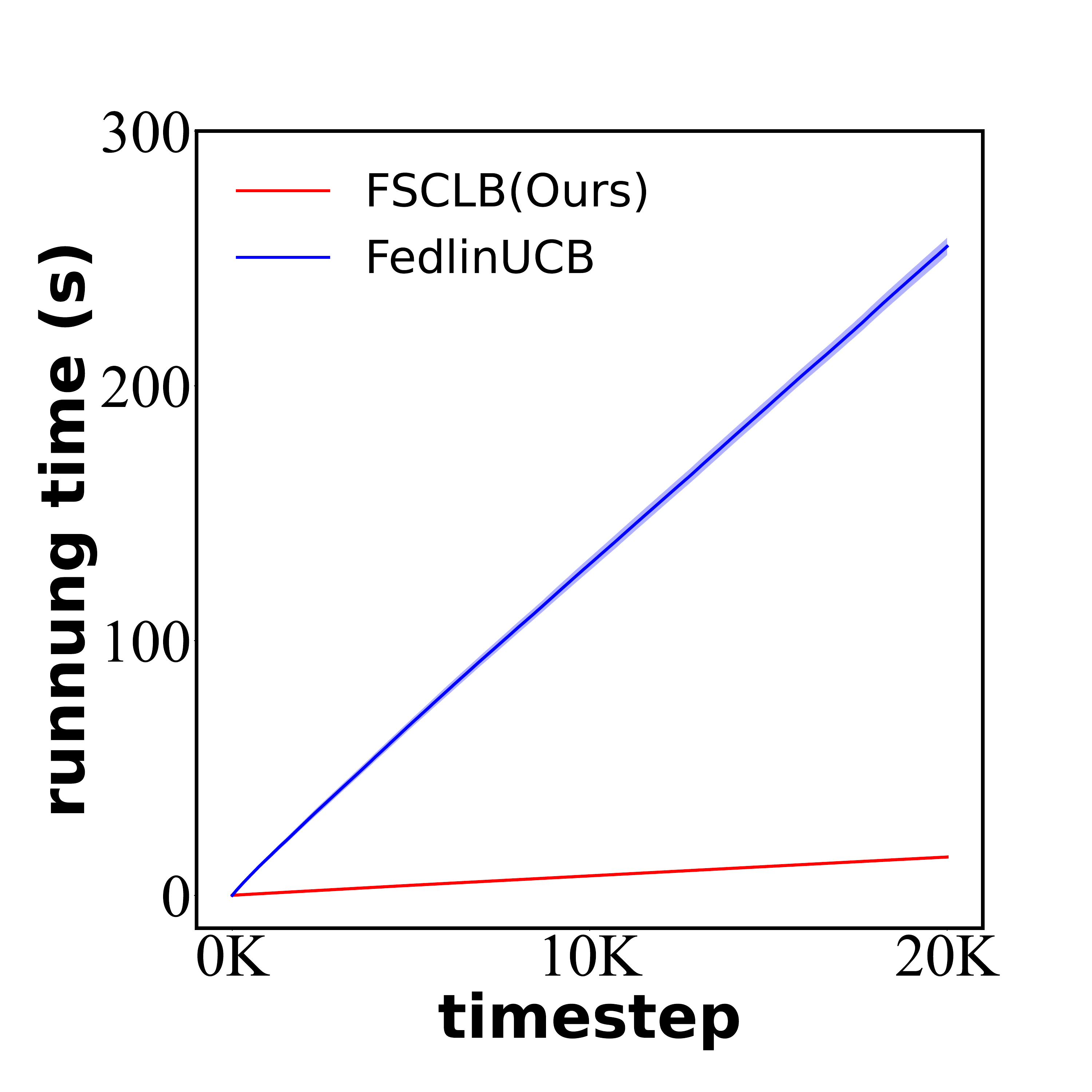}  
  \caption{satimage}
\end{subfigure}
\begin{subfigure}{.23\textwidth}
  \centering
  \includegraphics[width=\linewidth,height=1\textwidth]{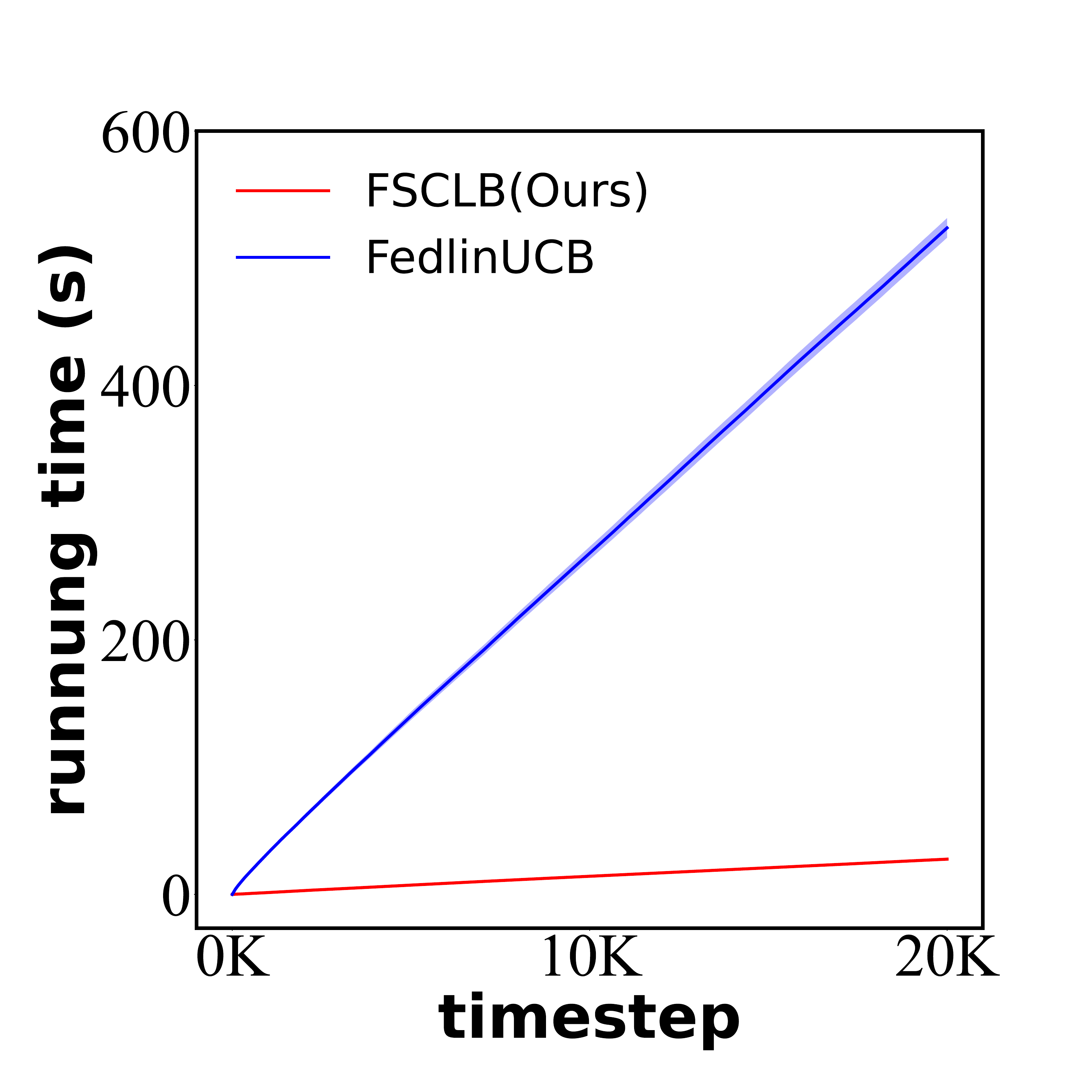}  
  \caption{mfeat}
\end{subfigure}
\caption{Computation cost in real dataset}
\label{fig:real_time}
\end{figure}
\begin{figure}[!ht]
\centering
\begin{subfigure}{.23\textwidth}
  \centering
  \includegraphics[width=\linewidth,height=1\textwidth]{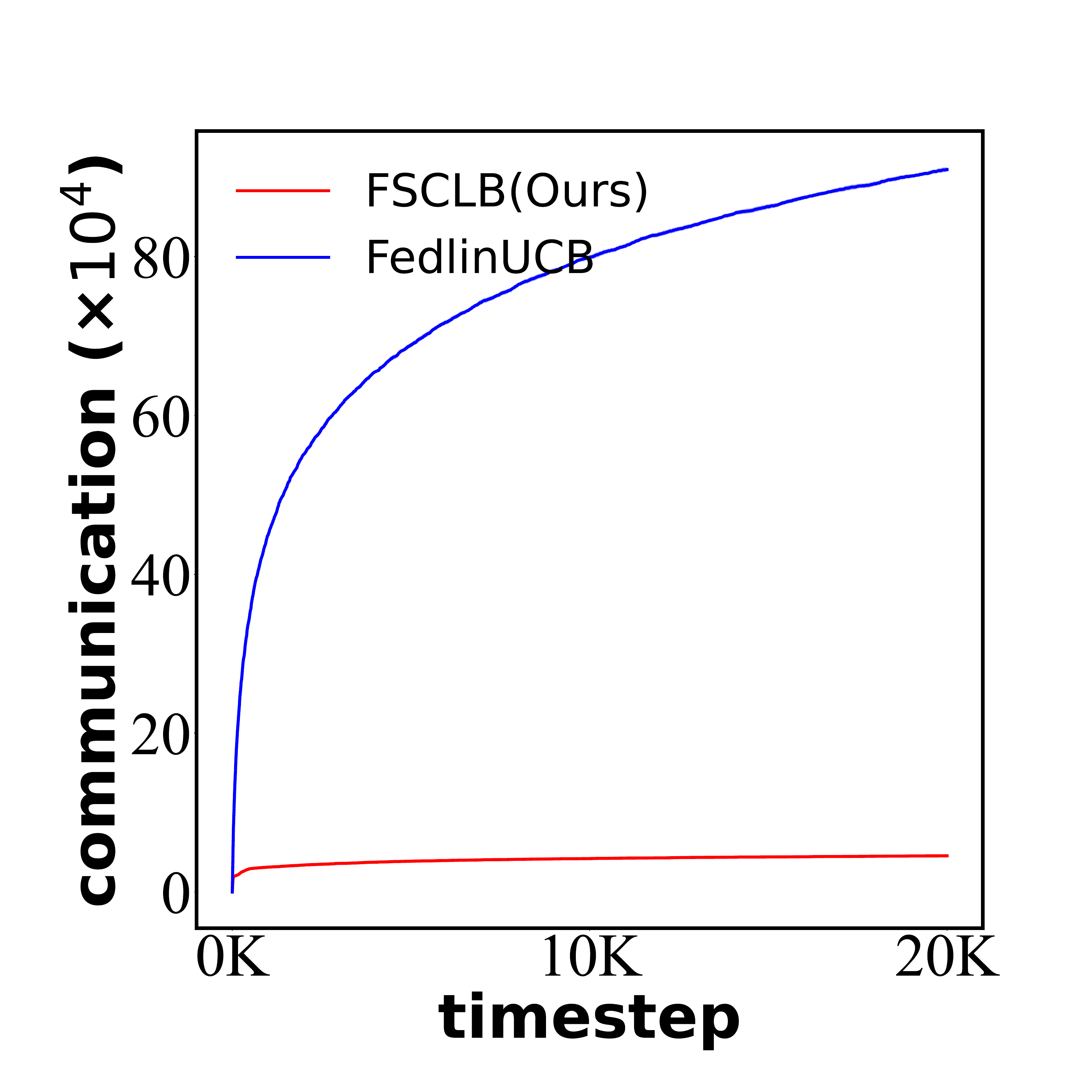}  
  \caption{satimage}
\end{subfigure}
\begin{subfigure}{.23\textwidth}
 \centering
  \includegraphics[width=\linewidth,height=1\textwidth]{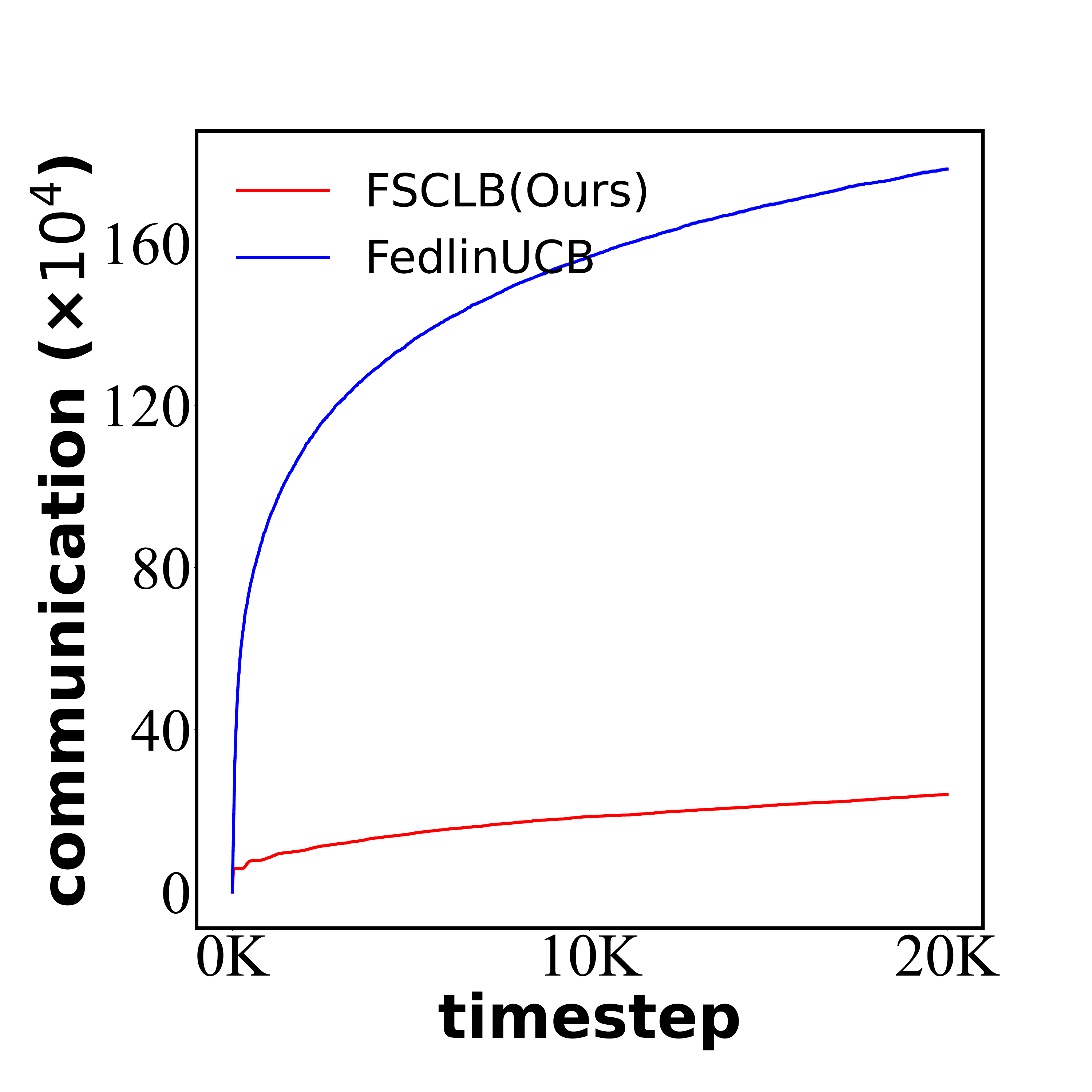}  
  \caption{mfeat}
\end{subfigure}
\caption{Communication in real dataset}
\label{fig:real_communication}
\end{figure}

\section{Conclusion}
In this work we consider the computational and communication cost in federated contextual linear bandits. We involve sketching into this setting and propose the FSCLB algorithm, which cuts local computation and communication. FSCLB employs a double-sketch strategy to guarantee low cost in both upload and download. By adopting SCFD \cite{SCFD} and building the communication trigger rule on the sketched matrix, we avoid the loss of asynchrony that classical FD can cause. We prove that FSCLB achieves a regret bound of $\widetilde{O}((\sqrt{d}+\sqrt{M\varepsilon_l})\sqrt{lT})$; once the sketch size $l$ exceeds the rank of the arm-selection sequence, the bound improves to $\widetilde{O}(\sqrt{ldT})$—matching the best known rate of non-sketched federated contextual linear bandits.  Moreover, FSCLB needs $O(l^2d)$ per-round computation and $O(ld)$ per-round communication ((when it occurs), improving on the $O(d^3)$ computation and $O(d^2)$ communication of prior work.
\newpage
\bibliographystyle{unsrt}  
\bibliography{references}

\newpage
\appendix
\onecolumn
\section{Appendix}
\subsection{Notation in Proof}

According to \cite{FedLinUCB}, for a squence $\bm{X}_t=[\bm{x}_1,...,\bm{x}_t]^{\top}$ in FSCLB, we define
\[
\bm{\Sigma}_t^{all}=\lambda \bm{I}+\sum_{i=1}^{t}\bm{x}_i\bm{x}_i^{\top}=\lambda\bm{I}+\bm{X}_T^{\top}\bm{X}_T.
\]
\[
\widetilde{\bm{V}}_t^{all}=\lambda \bm{I}+\widetilde{\bm{\Sigma}}_t^{ser}+\rho_t^{ser}\bm{I}+\sum_{m'=1}^{M}\left((\bm{S}_{m',t}^{loc})^{\top}\bm{S}_{m',t}^{loc}+\rho_{m',t}^{loc}\bm{I}\right).
\]
\[
\bm{b}_t^{all}=\sum_{i=1}^{t}r_i\bm{x}_i,\ \bm{y}_t^{all}=\sum_{i=1}^{t}\eta_i\bm{x}_i.
\]

We use  $N_m(t)$ to denote the set of moments when user $m$ communicated with the server before time $t$. And $N_{m,t}=\max\{N_m(t)\}$, $N_t=\max_{m'\in [M]}N_{m',t}$
\[
\widetilde{\bm{V}}_{m,t}^{up}=\sum_{i\in N_m(t)}\left((\bm{S}_{m,i}^{loc})^{\top}\bm{S}_{m,i}^{loc}+\rho_{m,i}^{loc}\bm{I}\right).
\] 
\[
\bm{\Sigma}_{m,t}^{up}=\sum_{j=1,m_j=m}^{N_{m,t}}\bm{x}_j\bm{x}_j^{\top},\ \bm{b}_{m,t}^{up}=\sum_{j=1,m_j=m}^{N_{m,t}}r_j\bm{x}_j,\ \bm{y}_{m,t}^{up}=\sum_{j=1,m_j=m}^{N_{m,t}}\eta_j\bm{x}_j.
\]

Then we can get
\[
\widetilde{\bm{\Sigma}}_t^{ser}+\rho_t^{ser}\bm{I}=\sum_{m=1}^{M}\widetilde{\bm{V}}_{m,t}^{up},\  \widetilde{\bm{\Sigma}}_t^{ser}\preceq \bm{X}_T^{\top}\bm{X}_T.
\]

Correspondingly, the data represented by the sketch on each user's local device is denoted as
\[
\bm{\Sigma}_{m,t}^{loc}=\sum_{j=N_{m,t}+1,m_j=m}^{t}\bm{x}_j\bm{x}_j^{\top},\ \bm{b}_{m,t}^{loc}=\sum_{j=N_{m,t}+1,m_j=m}^{t}r_j\bm{x}_j, \ \bm{y}_{m,t}^{loc}=\sum_{j=N_{m,t}+1,m_j=m}^{t}\eta_j\bm{x}_j.
\]

We use $\rho_t^{FD}$ to denote the minimum singular value of the same sequence at time  $t$ when running the sequence $X^{ser}_T$ under classical FD sketch. We define the sketch error $\bar{\rho}_t=\sum_{i=1}^{t}\rho_i^{FD}$ under classical FD sketch.

\subsection{Technical Lemmas}

This section collects the key, previously established properties and lemmas that we will use in our proofs.
\begin{lemma}
\label{Weyl inequality}
(Weyl inequality) For any positive semi-definite matrix $\bm{A},\bm{B}$, if $\bm{A}\succeq \bm{B}$ and $rank(\bm{A})=rank(\bm{B})$, then $\sigma_{min}(\bm{A})\geq \sigma_{min}(\bm{B})$
\end{lemma}
\begin{lemma}
\label{tech1}
For any positive semi-definite matrix $\bm{A}$ and any vector $\bm{x},\bm{y}$, we have $|\bm{x}^{\top}\bm{y}|\leq \left \|\bm{x}\right \|_{\bm{A}}\left \|\bm{y}\right \|_{\bm{A}^{-1}}$.
\end{lemma}
\begin{proof}
\begin{align}
|\bm{x}^{\top}\bm{y}|
&=\bm{x}^{\top}\bm{A}^{\frac{1}{2}}\bm{A}^{-\frac{1}{2}}\bm{y}\\
&\leq \left \|\bm{A}^{\frac{1}{2}}\bm{x} \right \|_2\left \|\bm{A}^{-\frac{1}{2}}\bm{y} \right \|_2 \text{ (Cauchy-Schwartz inequality)}\\
&=\left \|\bm{x}\right \|_{\bm{A}}\left \|\bm{y}\right \|_{\bm{A}^{-1}}.
\end{align}
\end{proof}
\begin{lemma}
By \cite{OFUL} and \cite{SOFUL}, with sketch size $l$, we have
\begin{align}
\nonumber
\sum_{t=1}^{T}\min\{1,\left \|\bm{x}_t \right \|_{(\bm{\Sigma}_{t-1}^{all})^{-1}}^2\} 
&\leq 2\ln\frac{det(\bm{\Sigma}_T^{all})}{det(\lambda\bm{I})}\\
\nonumber
&\leq 2d\ln\left (1+\varepsilon_l \right )+2l\ln\left (1+\frac{TL^2}{l\lambda} \right ).
\end{align}
\end{lemma}
\subsection{SCFD under Federated Contextual Linear Bandits}

The error term of SCFD is tied to the sequence it generates; this section will link the error produced by SCFD in federated learning to that in the traditional single-agent case.
\begin{lemma}
\label{Monotonicity in Federated Learning}
(Monotonicity in Federated Learning) Assume that the communication happened at $t_1,...,t_n$, then
\[
\widetilde{\bm{V}}_{t_j}^{ser}\succeq \widetilde{\bm{V}}_{t_i}^{ser}, \forall j\geq i.
\]
\end{lemma}
\begin{proof}
According to \cite{SCFD}, at $t_j$, we have
\[
(\widetilde{\bm{S}}_t^{ser})^{\top}\widetilde{\bm{S}}_t^{ser}+\widetilde{\sigma}_t^{ser}\bm{I}\succeq (\widetilde{\bm{S}}_{t-1}^{ser})^{\top}\widetilde{\bm{S}}_{t-1}^{ser}.
\]

Because $\widetilde{\rho}^{ser}_t=\sum_{i=1}^{t}\widetilde{\sigma}_i^{ser}$, we get
\[
(\widetilde{\bm{S}}_t^{ser})^{\top}\widetilde{\bm{S}}_t^{ser}+\widetilde{\rho}^{ser}_t\bm{I}\succeq (\widetilde{\bm{S}}_{t-1}^{ser})^{\top}\widetilde{\bm{S}}_{t-1}^{ser}+\widetilde{\rho}^{ser}_{t-1}\bm{I}.
\]

On the other hand, $\rho^{ser}_t \geq \rho^{ser}_{t-1}$, therefore
\[
(\widetilde{\bm{S}}_t^{ser})^{\top}\widetilde{\bm{S}}_t^{ser}+\Delta^{ser}_t\bm{I}\succeq (\widetilde{\bm{S}}_{t-1}^{ser})^{\top}\widetilde{\bm{S}}_{t-1}^{ser}+\Delta^{ser}_{t-1}\bm{I},
\]

this implies
\[
\widetilde{\bm{V}}_{t}^{ser}\succeq \widetilde{\bm{V}}_{t-1}^{ser}.
\]
\end{proof}
\begin{lemma}
\label{FL SCFD inequality}
For any $t\in [T]$, we can get
\[
(\lambda+\Delta_t^{ser})\bm{I}+\bm{X}_t^{\top}\bm{X}_t\succeq\widetilde{\bm{V}}_{t}^{ser}=\lambda\bm{I}+\sum_{m=1}^{M}\widetilde{\bm{V}}_{m,t}^{up}\succeq \lambda\bm{I}+\sum_{m=1}^{M}\bm{\Sigma}_{m,t}^{up}.
\]
\end{lemma}
\begin{proof}
According to \cite{SCFD}, we have
\begin{align}
\widetilde{\bm{V}}_{t}^{ser}\nonumber
&=(\lambda+\Delta_t^{ser})\bm{I}+(\widetilde{\bm{S}}_t^{ser})^{\top}\widetilde{\bm{S}}_t^{ser}\\ \nonumber
&=\lambda \bm{I}+\rho^{ser}_t\bm{I}+\widetilde{\rho}^{ser}_t\bm{I}+(\widetilde{\bm{S}}_t^{ser})^{\top}\widetilde{\bm{S}}_t^{ser}\\
&\succeq \lambda\bm{I}+\rho^{ser}_t\bm{I}+\widetilde{\bm{\Sigma}}_t^{ser}\\
&=\lambda\bm{I}+\sum_{m=1}^{M}\widetilde{\bm{V}}_{m,t}^{up}\\
&\succeq \lambda\bm{I}+\sum_{m=1}^{M}\bm{\Sigma}_{m,t}^{up},
\end{align}

and
\begin{align}
\widetilde{\bm{V}}_{t}^{ser}\nonumber
&=(\lambda+\Delta_t^{ser})\bm{I}+(\widetilde{\bm{S}}_t^{ser})^{\top}\widetilde{\bm{S}}_t^{ser}\\
&\preceq (\lambda+\Delta_t^{ser})\bm{I}+\bm{X}_t^{\top}\bm{X}_t.
\end{align}
\end{proof}
\begin{lemma}
\label{double sketch error}
According to \cite{SOFUL}, we define $\varepsilon_l=\min_{k\leq l-1}\frac{\lambda_{d-k}+\lambda_{d-k+1}+...+\lambda_d}{\lambda(l-k)} $, where $\lambda_1>...>\lambda_d$ are the eigenvalues of the correlation matrix $\bm{X}_T^{\top}\bm{X}$. Let $B_c$ be the number of communication. Then, 
\[
\frac{\rho_T^{ser}}{\lambda}<B_c\varepsilon_l,\ \frac{\widetilde{\rho}_T^{ser}}{\lambda}<\varepsilon_l.
\]
\end{lemma}
\begin{proof}
We define $\varepsilon_l^S=\min_{k\leq l-1}\frac{\lambda_{d-k}^S+\lambda_{d-k+1}^S+...+\lambda_d^S}{\lambda(l-k)}$, where $\lambda_1^S>...>\lambda_d^S$ are the eigenvalues of the sketch matrix $\widetilde{\bm{\Sigma}}_T^{ser}$. According to \cite{SOFUL}, $\frac{\widetilde{\rho}_T^{ser}}{\lambda}\leq \varepsilon_l^S$. Because $\widetilde{\bm{\Sigma}}_T^{ser}\preceq \bm{X}_T^{\top}\bm{X}$, according to Lemma \ref{Weyl inequality}, we known that $\varepsilon_l^S\leq \varepsilon_l$. Then $\frac{\widetilde{\rho}_T^{ser}}{\lambda}\leq \varepsilon_l$.
\begin{align}
\rho_t^{ser}=\rho_{N_t}^{ser}=\sum_{i\in \cup_{m=1}^{M}N_m(t)}\rho_{m_i,i}^{loc}.
\end{align}

We partition the sequence $X_T^{ser}$ according to the order of timestamps sorted in ascending order within $\cup_{m=1}^{M}N_m(t)$. Then, within each epoch, it is equivalent to a single user running the classical FD algorithm, and each epoch starts anew. For any $t_k\in \cup_{m=1}^{M}N_m(t)$, assume $t_k<t<t_{k+1}$, and there is only one user activated in this epoch. We denote by $\varepsilon_l^{(k)}$ the value of $\varepsilon_l$ that would be obtained by running the FD sketch on the sequence from round $t_k$ to $t_{k+1}$ with single-agent.

Because $\bm{X}_{[t_k,...,t_k-1]}^{\top} \bm{X}_{[t_k,...,t_k-1]}\preceq \bm{X}_{T}^{\top} \bm{X}_{T}$, Lemma \ref{Weyl inequality} implies that when the singular values are arranged in descending order, each singular value of the former matrix is at most the corresponding one of the latter; then we get $\varepsilon_l^{(k)}\leq \varepsilon_l$.

Note that in FSCLB the local sketch matrix is reset to zero whenever a communication round ends, according to \cite{SOFUL}, we have $\rho_{m,t_k}^{loc}<\lambda \varepsilon_l^{(k)}$. We known that $B_c=|\cup_{m=1}^{M}N_m(t)|$, therefore
\[
\rho_t^{ser}=\sum_{k=1}^{B_c}\rho_{m_{t_k},t_k}^{loc}\leq B_c\lambda\varepsilon_l^{(k)}\leq \lambda B_c\varepsilon_l.
\]
\end{proof}
\subsection{Analysis fo Communication and Computation}

\subsubsection{Proof for Communication}

Our proof of the communication bound is a sketch version of Appendix B.1 in \cite{FedLinUCB}. Firstly, we define
\[
\tau_i=min\{t|det(\bm{\widetilde{V}}_{t}^{ser})\geq 2^{i}\lambda^{d} \}.
\]

For each agent $m$, let $n_m$ be the number of communications triggered by the determinant of agent $m$ during this epoch $\tau_i\sim\tau_{i+1}-1$, and we denote the communication rounds as $t_1,...,t_{n_m}$ for simplicity.

We can get
\[
det(\bm{\widetilde{V}}_{m,t_j}+(\bm{S}_{m,t_j}^{loc})^{\top}\bm{S}_{m,t_j}^{loc}+\rho_{m,t_j}^{loc}\bm{I})-det(\bm{\widetilde{V}}_{m,t_j})>\alpha det(\bm{\widetilde{V}}_{m,t_j}).
\]

Since $\bm{\widetilde{V}}_{m,t_j}\succeq \bm{\widetilde{V}}_{\tau_i}^{ser}$, which implies that
\[
det(\bm{\widetilde{V}}_{\tau_i}^{ser}+(\bm{S}_{m,t_j}^{loc})^{\top}\bm{S}_{m,t_j}^{loc}+\rho_{m,t_j}^{loc}\bm{I})-det(\bm{\widetilde{V}}_{\tau_i}^{ser})>\alpha det(\bm{\widetilde{V}}_{\tau_i}^{ser}).
\]

According to $\bm{\widetilde{V}}_{m,t_j}\succeq \bm{\widetilde{V}}_{\tau_i}^{ser}$ and Lemma D.1 in \cite{FedLinUCB}, we have
\[
det(\bm{\widetilde{V}}_{m,t_j}+(\bm{S}_{m,t_j}^{loc})^{\top}\bm{S}_{m,t_j}^{loc}+\rho_{m,t_j}^{loc}\bm{I})-det(\bm{\widetilde{V}}_{m,t_j})>det(\bm{\widetilde{V}}_{\tau_i}^{ser}+(\bm{S}_{m,t_j}^{loc})^{\top}\bm{S}_{m,t_j}^{loc}+\rho_{m,t_j}^{loc}\bm{I})-det(\bm{\widetilde{V}}_{\tau_i}^{ser}).
\]

In addition, we define the sequence of all communications from $\tau_i$ to $\tau_{i+1}-1$ as $t'_1,...,t'_L$. We have
\begin{align}
det(\bm{\widetilde{V}}_{t'_j}^{ser})-det(\bm{\widetilde{V}}_{t'_{j-1}}^{ser})
&=det(\bm{\widetilde{V}}_{t'_{j-1}}^{ser}+(\bm{S}_{m_{t'_j},t'_j}^{loc})^{\top}\bm{S}_{m_{t'_j},t'_j}^{loc}+\rho_{m_{t'_j},t'_j}^{loc}\bm{I})-det(\bm{\widetilde{V}}_{t'_{j-1}}^{ser})\\
&\geq det(\bm{\widetilde{V}}_{\tau_i}^{ser}+(\bm{S}_{m_{t'_j},t'_j}^{loc})^{\top}\bm{S}_{m_{t'_j},t'_j}^{loc}+\rho_{m_{t'_j},t'_j}^{loc}\bm{I})-det(\bm{\widetilde{V}}_{\tau_i}^{ser})\\
&>\alpha det(\bm{\widetilde{V}}_{\tau_i}^{ser}).
\end{align}

Then 
\begin{align}
det(\bm{\widetilde{V}}_{\tau_{i+1}-1}^{ser})-det(\bm{\widetilde{V}}_{\tau_i}^{ser})=\sum_{1\leq j\leq L}det(\bm{\widetilde{V}}_{t'_j}^{ser})-det(\bm{\widetilde{V}}_{t'_{j-1}}^{ser})\geq \sum_{m=1}^{M}(n_m-1)\alpha det(\bm{\widetilde{V}}_{\tau_i}^{ser}).
\end{align}

Since $det(\bm{\widetilde{V}}_{\tau_{i+1}-1}^{ser})<2det(\bm{\widetilde{V}}_{\tau_i}^{ser})$, we have
\[
\sum_{m=1}^{M}n_m \leq M+2/\alpha.
\]

Then the number of communications within one epoch is bounded by $2(M+1/\alpha)$.
Since $\bm{\widetilde{V}}_{T}^{ser}$ is positive definite, and $\bm{\widetilde{V}}_{T}^{ser}\preceq (\lambda+\Delta_t^{ser})\bm{I}+\bm{X}_t^{\top}\bm{X}_t$, by the arithmetic-geometric inequality, we have
\begin{align}
det(\bm{\widetilde{V}}_{T}^{ser})
&\leq \left(\frac{tr(\bm{\widetilde{V}}_{T}^{ser})}{d} \right)^d \\
&\leq \left(\frac{1}{d}tr((\lambda+\Delta_T^{ser})\bm{I}+\sum_{t=1}^{T}\bm{x}\bm{x}^{\top}) \right)^d \\
&\leq (\lambda+\Delta_T^{ser})^d\left(1+\frac{TL^2}{\lambda d} \right)^d.
\end{align}

Then we have
\[
max\{i\geq 0| \tau_i\ne \phi\}=\log_2\frac{det(\bm{\widetilde{V}}_{T}^{ser})}{\lambda^d}\leq d \log_2(1+\frac{\Delta_T^{ser}}{\lambda})\left(1+\frac{TL^2}{\lambda d} \right).
\]

We have $\Delta_T^{ser}<(T+1)\varepsilon_l$, then the number of communications triggered by the determinant is bounded by
\[
2(M+2/\alpha)d\log(1+\varepsilon_l)\left(1+\frac{ TL^2}{\lambda d}\right).
\]

Therefore, the total communication times is bounded by $2d(M+1/\alpha)\log (1+\varepsilon_l)\left(1+\frac{ TL^2}{\lambda d}\right)$.

The total computation cost is bounded by $O(\min\{2l,d\}^2dT+(M+1/\alpha)d^4\log (1+\varepsilon_l)T)$.

The total communication cost is bounded by $O(ld^2(M+1/\alpha)\log (1+\varepsilon_l)T)$.

\subsection{Proof for Covariance Comparison}
\begin{lemma}
\label{sketch covariance comparison}
(sketch covariance comparison) For any time $t\in [T]$ and each $m \in [M]$
\[
\bm{\widetilde{V}}_{t}^{ser}\succeq \frac{1}{\alpha}\left((\bm{S}_{m,t}^{loc})^{\top}\bm{S}_{m,t}^{loc}+\rho_{m,t}^{loc}\bm{I}\right).
\]

If user $m$ is the only active agent from round $t_1$ to $t_2-1$ and user $m$ only communicated with the server at round $t_1$, then for all $t_1+1\leq t \leq t_2$, it further holds that
\[
\bm{\widetilde{V}}_{m,t} \succeq \frac{1}{1+M\alpha}\bm{\Sigma}_t^{all}
\]
\end{lemma}
\begin{proof}
According to Section B.2 in \cite{FedLinUCB}, it is obciously that
\[
\bm{\widetilde{V}}_{m,t} \succeq \frac{1}{\alpha}\left((\bm{S}_{m,t}^{loc})^{\top}\bm{S}_{m,t}^{loc}+\rho_{m,t}^{loc}\bm{I}\right),
\]
\[ 
\bm{\widetilde{V}}_{t}^{ser} \succeq \bm{\widetilde{V}}_{m,t} \succeq \frac{1}{\alpha}\left((\bm{S}_{m,t}^{loc})^{\top}\bm{S}_{m,t}^{loc}+\rho_{m,t}^{loc}\bm{I}\right).
\]

For the second inequality in lemma
\[
\bm{\widetilde{V}}_{t}^{ser}\succeq \frac{1}{M\alpha}\sum_{m'=1}^{M}\left((\bm{S}_{m',t}^{loc})^{\top}\bm{S}_{m',t}^{loc}+\rho_{m',t}^{loc}\bm{I}\right),
\]

and
\begin{align}
\bm{\widetilde{V}}_{m,t}
&=\bm{\widetilde{V}}_{t}^{ser}\\
&\succeq \frac{1}{1+M\alpha}\left(\bm{\widetilde{V}}_{t}^{ser}+ \sum_{m'=1}^{M}\left((\bm{S}_{m',t}^{loc})^{\top}\bm{S}_{m',t}^{loc}+\rho_{m',t}^{loc}\bm{I}\right)\right)\\
&\succeq \frac{1}{1+M\alpha}\left(\lambda\bm{I}+\sum_{m=1}^{M}\bm{\Sigma}_{m,t}^{up}+ \sum_{m'=1}^{M}\left((\bm{S}_{m',t}^{loc})^{\top}\bm{S}_{m',t}^{loc}+\rho_{m',t}^{loc}\bm{I}\right)\right)\\
&\succeq \frac{1}{1+M\alpha}\left(\lambda\bm{I}+\sum_{m=1}^{M}\bm{\Sigma}_{m,t}^{up}+ \sum_{m'=1}^{M}\bm{\Sigma}_{m,t}^{loc}\right)\\
&= \frac{1}{1+M\alpha}\bm{\Sigma}_t^{all}.
\end{align}

The last inequality is because Proposition 2 in \cite{SOFUL}.
\end{proof}

\subsection{Proof of Lemma 4.1}
\begin{lemma}
\label{federated sketch confidence bound}
(Lemma 4.1 and Lemma 5.1 in Section 4,5)
With probability at least $1-\delta$, for each $t$, the estimate satisfies that 
\[
\left \|\hat{\bm{\theta}}_{m,t+1}-\bm{\theta}^*  \right \|_{\bm{\widetilde{V}}_{m,t+1}}\leq \beta_{m,t}(\delta) \leq \beta.
\]
\end{lemma}
\begin{proof}
Because $\hat{\bm{\theta}}_{m,t},\bm{\widetilde{V}}_{m,t}$ update when $m$ communicated with server, then the proof is equal to proof that
\[
\left \|\hat{\bm{\theta}}_{m,N_{m_t,t}+1}-\bm{\theta}^*  \right \|_{\bm{\widetilde{V}}_{m,N_{m_t,t}+1}}\leq \beta_{m,t}(\delta) \leq \beta.
\]

Thus, we only need to consider for those round $t$ where user $m$ communicates with the server. 
\begin{align}
\left \|\hat{\bm{\theta}}_{m,t+1}-\bm{\theta}^*  \right \|_{\bm{\widetilde{V}}_{m,t+1}}^2
&=\left(\hat{\bm{\theta}}_{m,t+1}-\bm{\theta}^*\right)^{\top}\bm{\widetilde{V}}_{m,t+1}\left(\hat{\bm{\theta}}_{m,t+1}-\bm{\theta}^*\right)\\
&=\left(\hat{\bm{\theta}}_{m,t+1}-\bm{\theta}^*\right)^{\top}\bm{\widetilde{V}}_{m,t+1}\left(\bm{\widetilde{V}}_{m,t+1}^{-1}\sum_{m'=1}^{M}\bm{b}_{m',t}^{up}-\bm{\theta}^*\right)\\
&=\left(\hat{\bm{\theta}}_{m,t+1}-\bm{\theta}^*\right)^{\top}\bm{\widetilde{V}}_{m,t+1}\left(\bm{\widetilde{V}}_{m,t+1}^{-1}\sum_{m'=1}^{M}(\bm{\Sigma}_{m',t}^{up}\bm{\theta}^*+\bm{y}_{m',t}^{up})-\bm{\theta}^*\right)\\
&=\left(\hat{\bm{\theta}}_{m,t+1}-\bm{\theta}^*\right)^{\top}\bm{\widetilde{V}}_{m,t+1}\left(\bm{\widetilde{V}}_{m,t+1}^{-1}\sum_{m'=1}^{M}\bm{\Sigma}_{m',t}^{up}\bm{\theta}^*-\bm{\theta}^*\right)+\left(\hat{\bm{\theta}}_{m,t+1}-\bm{\theta}^*\right)^{\top}\sum_{m'=1}^{M}\bm{y}_{m',t}^{up}\\
&=\underset{(A)}{\underbrace{\left(\hat{\bm{\theta}}_{m,t+1}-\bm{\theta}^*\right)^{\top}\left(\sum_{m'=1}^{M}\bm{\Sigma}_{m',t}^{up}-\bm{\widetilde{V}}_{m,t+1}\right)\bm{\theta}^*}}+\underset{(B)}{\underbrace{\left(\hat{\bm{\theta}}_{m,t+1}-\bm{\theta}^*\right)^{\top}\sum_{m'=1}^{M}\bm{y}_{m',t}^{up}}}. 
\end{align}

For term (B)
\begin{align}
\left(\hat{\bm{\theta}}_{m,t+1}-\bm{\theta}^*\right)^{\top}\sum_{m'=1}^{M}\bm{y}_{m',t}^{up}
&=\left(\hat{\bm{\theta}}_{m,t+1}-\bm{\theta}^*\right)^{\top}\sum_{m'=1}^{M}(\bm{y}_{m',t}^{up}+\bm{y}_{m',t}^{loc})-\left(\hat{\bm{\theta}}_{m,t+1}-\bm{\theta}^*\right)^{\top}\sum_{m'=1}^{M}\bm{y}_{m',t}^{loc}\\
&=\left(\hat{\bm{\theta}}_{m,t+1}-\bm{\theta}^*\right)^{\top}\bm{y}_{t}^{all}-\left(\hat{\bm{\theta}}_{m,t+1}-\bm{\theta}^*\right)^{\top}\sum_{m'=1}^{M}\bm{y}_{m',t}^{loc}.
\end{align}

Then
\begin{align}
\left |\left(\hat{\bm{\theta}}_{m,t+1}-\bm{\theta}^*\right)^{\top}\sum_{m'=1}^{M}\bm{y}_{m',t}^{up}\right |
&\leq \left |\left(\hat{\bm{\theta}}_{m,t+1}-\bm{\theta}^*\right)^{\top}\bm{y}_{t}^{all}\right |+\left |\left(\hat{\bm{\theta}}_{m,t+1}-\bm{\theta}^*\right)^{\top}\sum_{m'=1}^{M}\bm{y}_{m',t}^{loc}\right |\\ \nonumber
&\leq \left \| \hat{\bm{\theta}}_{m,t+1}-\bm{\theta}^* \right \|_{\bm{\Sigma}_t^{all}} \left \| \bm{y}_{t}^{all} \right \|_{(\bm{\Sigma}_t^{all})^{-1}}\\ 
&+\sum_{m'=1}^{M}\left \| \hat{\bm{\theta}}_{m,t+1}-\bm{\theta}^* \right \|_{\alpha\lambda\bm{I}+\bm{\Sigma}_{m',t}^{loc}} \left \| \bm{y}_{m',t}^{loc} \right \|_{(\alpha\lambda\bm{I}+\bm{\Sigma}_{m',t}^{loc})^{-1}}.
\end{align}

The second inequality is because Lemma \ref{tech1}.

According to the Proposition 2 in \cite{SOFUL}, for any user $m$
\begin{align}
\bm{\Sigma}_{m,t}^{loc}\preceq (\bm{S}_{m,i}^{loc})^{\top}\bm{S}_{m,i}^{loc}+ \rho_{m,t}^{loc}\bm{I}.
\end{align}

Furthermore
\begin{align}
\Delta_{m,t+1}=\rho_t^{ser}+\widetilde{\rho}_t^{ser}\leq (B_c+1)\varepsilon_l\leq R_1\lambda(M+1/\alpha)d\varepsilon_l\log (1+\varepsilon_l)T,
\end{align}

where $R_1$ is a sufficient large constant. Then, combining with Lemma \ref{sketch covariance comparison}, we have
\begin{align}
\frac{\left \| \hat{\bm{\theta}}_{m,t+1}-\bm{\theta}^* \right \|_{\bm{\Sigma}_t^{all}}}{\left \| \hat{\bm{\theta}}_{m,t+1}-\bm{\theta}^* \right \|_{\bm{\widetilde{V}}_{m,t+1}}}\leq \sqrt{1+M\alpha},
\end{align}

and
\begin{align}
\frac{\left \| \hat{\bm{\theta}}_{m,t+1}-\bm{\theta}^* \right \|_{\alpha\lambda\bm{I}+\bm{\Sigma}_{m',t}^{loc}}}{\left \| \hat{\bm{\theta}}_{m,t+1}-\bm{\theta}^* \right \|_{\bm{\widetilde{V}}_{m,t+1}}}
&\leq \sqrt{\frac{\left \| \hat{\bm{\theta}}_{m,t+1}-\bm{\theta}^* \right \|_{(\bm{S}_{m,t+1}^{loc})^{\top}\bm{S}_{m,t+1}^{loc}+\rho_{m,t+1}^{loc}\bm{I}}^2+\alpha\lambda\left \| \hat{\bm{\theta}}_{m,t+1}-\bm{\theta}^* \right \|_2^2}{\left \| \hat{\bm{\theta}}_{m,t+1}-\bm{\theta}^* \right \|_{\bm{\widetilde{V}}_{m,t+1}}^2}}\\
&\leq \sqrt{\frac{\alpha\left \| \hat{\bm{\theta}}_{m,t+1}-\bm{\theta}^* \right \|_{\bm{\widetilde{V}}_{m,t+1}}^2+\alpha\lambda\left \| \hat{\bm{\theta}}_{m,t+1}-\bm{\theta}^* \right \|_2^2}{\left \| \hat{\bm{\theta}}_{m,t+1}-\bm{\theta}^* \right \|_{\bm{\widetilde{V}}_{m,t+1}}^2}}\\
&\leq \sqrt{\alpha+\frac{\alpha\lambda\left \| \hat{\bm{\theta}}_{m,t+1}-\bm{\theta}^* \right \|_2^2}{\left \| \hat{\bm{\theta}}_{m,t+1}-\bm{\theta}^* \right \|_{\lambda\bm{I}}^2}}\\
&=\sqrt{2\alpha}.
\end{align}

According to \cite{OFUL} and \cite{FedLinUCB}, with probability at least $1-\frac{\delta}{T^2M}$, we have
\begin{align}
\left \|\bm{y}_{t}^{all} \right \|_{(\bm{\Sigma}_t^{all})^{-1}}
&\leq 2R\sqrt{\ln\left (\frac{det(\bm{\Sigma}_t^{all})}{det(\lambda\bm{I})}\right )+\ln(\frac{1}{\delta})}\\
&\leq 2R\sqrt{d\log{\left (\frac{1+TL^2/(\alpha\lambda)}{\delta}\right )}}+\sqrt{\lambda}S,
\end{align}

and
\begin{align}
\left \| \bm{y}_{m',t}^{loc} \right \|_{(\alpha\lambda\bm{I}+\bm{\Sigma}_{m',t}^{loc})^{-1}}\leq 2R\sqrt{d\log{\left (\frac{1+TL^2/(\alpha\lambda)}{\delta}\right )}}+\sqrt{\lambda}S.
\end{align}

Define $\widetilde{\alpha}=\min\{\alpha,1\}$. Therefore, term (B) is bounded by
\begin{align}
&\left |\left(\hat{\bm{\theta}}_{m,t+1}-\bm{\theta}^*\right)^{\top}\sum_{m'=1}^{M}\bm{y}_{m',t}^{up}\right |\\
&\leq \left(\sqrt{1+M\alpha}+ M\sqrt{2\alpha}\right)\left ( 2R\sqrt{d\log{\left (\frac{1+TL^2/(\widetilde{\alpha}\lambda)}{\delta}\right )}}+\sqrt{\lambda}S  \right )\left \| \hat{\bm{\theta}}_{m,t+1}-\bm{\theta}^* \right \|_{\bm{\widetilde{V}}_{m,t+1}}.
\end{align}

For term (A), by Cauchy-Schwartz inequality and the triangle inequality, we have
\begin{align}
&\left |\left(\hat{\bm{\theta}}_{m,t+1}-\bm{\theta}^*\right)^{\top}\left(\sum_{m'=1}^{M}\bm{\Sigma}_{m',t}^{up}-\bm{\widetilde{V}}_{m,t+1}\right)\bm{\theta}^* \right |\\
&\leq \left \|\hat{\bm{\theta}}_{m,t+1}-\bm{\theta}^*\right \|_{\bm{\widetilde{V}}_{m,t+1}} \left \|\left(\sum_{m'=1}^{M}\bm{\Sigma}_{m',t}^{up}-\bm{\widetilde{V}}_{m,t+1}\right)\bm{\theta}^* \right \|_{\bm{\widetilde{V}}_{m_t,t}^{-1}}\\
&\leq \frac{1}{\sqrt{\lambda}}\left \|\hat{\bm{\theta}}_{m,t+1}-\bm{\theta}^*\right \|_{\bm{\widetilde{V}}_{m,t+1}} \left \|\left(\sum_{m'=1}^{M}\bm{\Sigma}_{m',t}^{up}-\bm{\widetilde{V}}_{m,t+1}\right)\right \|_{\bm{\widetilde{V}}_{m,t+1}^{-1}}\left \|\bm{\theta}^* \right \|_{\bm{\widetilde{V}}_{m,t+1}^{-1}}\\
&\leq \frac{1}{\sqrt{\lambda}} \left \|\bm{\theta}^* \right \|_{2}\left \|\hat{\bm{\theta}}_{m,t+1}-\bm{\theta}^*\right \|_{\bm{\widetilde{V}}_{m,t+1}} \left \|\sum_{m'=1}^{M}\bm{\Sigma}_{m',t}^{up}-\bm{\widetilde{V}}_{m,t+1} \right \|_{2}\\
&\leq \frac{S}{\sqrt{\lambda}}\left \|\hat{\bm{\theta}}_{m,t+1}-\bm{\theta}^*\right \|_{\bm{\widetilde{V}}_{m,t+1}} \left (\left \|\sum_{m'=1}^{M}\bm{\Sigma}_{m',t}^{up}-(\bm{\widetilde{V}}_{m,t+1}-\lambda \bm{I}) \right \|_{2}+\lambda\right ).
\end{align}

We have
\begin{align}
\bm{\widetilde{V}}_{m,t+1}-\lambda \bm{I} \preceq \sum_{m'=1}^{M}\bm{\Sigma}_{m',t}^{up}\preceq \bm{\widetilde{V}}_{m,t+1}-\lambda \bm{I}+\Delta_{m,t+1}\bm{I}, \forall m\in [M].
\end{align}

Then
\begin{align}
&\left |\left(\hat{\bm{\theta}}_{m,t+1}-\bm{\theta}^*\right)^{\top}\left(\sum_{m'=1}^{M}\bm{\Sigma}_{m',t}^{up}-\bm{\widetilde{V}}_{m,t+1}\right)\bm{\theta}^* \right |\\
&\leq \frac{(\sqrt{\Delta_{m,t+1}}+\lambda)S}{\sqrt{\lambda}}\left \|\hat{\bm{\theta}}_{m,t+1}-\bm{\theta}^*\right \|_{\bm{\widetilde{V}}_{m,t+1}}.
\end{align}

Combining term (A) and (B), we get
\begin{align}
\left \|\hat{\bm{\theta}}_{m,t}-\bm{\theta}^*  \right \|_{\bm{\widetilde{V}}_{m,t}}
&\leq \left (\sqrt{1+M\alpha}+M\sqrt{2\alpha}\right )\left ( R\sqrt{d\log{\left (\frac{1+TL^2/(\widetilde{\alpha}\lambda)}{\delta}\right )}}+\sqrt{\lambda}S \right )\\
&+(\sqrt{\lambda}+\sqrt{\frac{\Delta_{m,t+1}}{\lambda}})S\\ \nonumber
&=\beta_{m,t}(\delta).
\end{align}

Let
\begin{align}
\beta
&=\left (\sqrt{1+M\alpha}{+}M\sqrt{2\alpha}\right )\left ( R\sqrt{d\log{\left (\frac{1+TL^2/(\widetilde{\alpha}\lambda)}{\delta}\right )}}{+}\sqrt{\lambda}S \right ){+}(\sqrt{R_1(M+1/\alpha)d\varepsilon_l\log (1+\varepsilon_l)T}{+}\sqrt{\lambda})S.
\end{align}

Then
\[
\beta_{m,t}(\delta)\leq \beta=\widetilde{O}\left((1+M\sqrt{\alpha}+\sqrt{(1/\alpha+M)\varepsilon_l})\sqrt{d}\right).
\]
\end{proof}
\begin{lemma}
\label{per-round regret}
For per-round regret, we have
\[
\max_{\bm{x}\in D_t}\left \langle \bm{\theta}^*,\bm{x}\right \rangle-\left \langle \bm{\theta}^*,\bm{x}_t \right \rangle \leq 2\beta \left \|\bm{x}_t \right \|_{\bm{\widetilde{V}}_{m_t,t}^{-1}}.
\]
\end{lemma}

According to the proof of Lemma B.2 in \cite{FedLinUCB}, this lemma is obviously.
\begin{lemma}
\label{Sketched transform}
When the communication does not happened, it holds that
\[
\left \|\bm{x}_t \right \|_{\bm{\widetilde{V}}_{m_t,t}^{-1}} \leq \sqrt{1+M\alpha}\left \|\bm{x}_t \right \|_{(\bm{\Sigma}_{t}^{all})^{-1}}.
\]
\end{lemma}

This Lemma is because Lemma \ref{sketch covariance comparison}.
\subsection{Proof of Theorem 5.1}

We assume that the sequence of rounds that the active agent communicates with server is $0=t_0<t_1<...<t_N=T+1$, and from $t_i+1$ to $t_{i+1}-1$ there is only one agent active, i.e. $m_{t_i}=m_{t_i+1}=...=m_{t_{i+1}-1}$. According to Lemma \ref{per-round regret}, we can refined the regret as
\begin{align}
Regret(T)
&=\sum_{t=1}^{T}\left ( \max_{\bm{x}\in D_t} \left \langle \bm{\theta}^*,\bm{x} \right \rangle- \left \langle \bm{\theta}^*,\bm{x}_t \right \rangle \right )\\
&\leq \sum_{i=0}^{N-1}\sum_{t=t_i+1}^{t_{i+1}-1}2\beta \left \|\bm{x}_t \right \|_{\bm{\widetilde{V}}_{m_t,t}^{-1}}+\sum_{i=0}^{N-1}\min\{\max_{\bm{x}\in D_{t_i}} \left \langle \bm{\theta}^*,\bm{x} \right \rangle- \left \langle \bm{\theta}^*,\bm{x}_{t_i} \right \rangle,2\beta \left \|\bm{x}_{t_i} \right \|_{\bm{\widetilde{V}}_{m_{t_i},t_i}^{-1}}\}.
\end{align}

For a more refined analysis of the rounds $\{t_i\}_{i=1}^{N}$, we define
\[
T_i=min\{t|det(\bm{\Sigma}_{t}^{all})\geq 2^{i}\lambda^{d} \}.
\]

and let $N'$ be the largest integer such that $T_{N'}$ is not empty. For each time interval from $T_i$ to $T_{i+1}$
and each user $m\in [M]$, suppose $m$ communicates with the server more than once, where the communications occur at rounds $T_i\leq T_{i,1}<T_{i,2}<...<T_{i,j}<T_{i+1}$. Then
\begin{align}
2\beta \left \|\bm{x}_{T_{i,j}} \right \|_{\bm{\widetilde{V}}_{m,T_{i,j}}^{-1}}\leq 2\beta \left \|\bm{x}_{T_{i,j}} \right \|_{\bm{\widetilde{V}}_{m,T_{i,j-1}+1}^{-1}}\leq 2\beta\sqrt{1+M\alpha}\left \|\bm{x}_{T_{i,j}} \right \|_{(\bm{\Sigma}_{m,T_{i,j-1}+1}^{all})^{-1}}.
\end{align}

On the other hands, according to our define, we known that 
\[
det(\bm{\Sigma}_{m,T_{i+1}-1}^{all})/det(\bm{\Sigma}_{m,T_{i,j-1}+1}^{all})<2.
\]

Then, we have
\begin{align}
2\beta \left \|\bm{x}_{T_{i,j}} \right \|_{\bm{\widetilde{V}}_{m,T_{i,j}}^{-1}}
&\leq 2\beta\sqrt{1+M\alpha}\left \|\bm{x}_{T_{i,j}} \right \|_{(\bm{\Sigma}_{m,T_{i,j-1}+1}^{all})^{-1}}\\
&\leq 2\beta\sqrt{2(1+M\alpha)}\left \|\bm{x}_{T_{i,j}} \right \|_{(\bm{\Sigma}_{m,T_{i+1}-1}^{all})^{-1}}\\
&\leq 2\beta\sqrt{2(1+M\alpha)}\left \|\bm{x}_{T_{i,j}} \right \|_{(\bm{\Sigma}_{m,T_{i,j}}^{all})^{-1}}.
\end{align}

According to \cite{FedLinUCB} (Section 6.2), we known that
\[
\sum_{i=0}^{N-1}\min\{\max_{\bm{x}\in D_{t_i}} \left \langle \bm{\theta}^*,\bm{x} \right \rangle- \left \langle \bm{\theta}^*,\bm{x}_{t_i} \right \rangle,2\beta \left \|\bm{x}_{t_i} \right \|_{\bm{\widetilde{V}}_{m_{t_i},t_i}^{-1}}\}\leq 2SLMN'+\sum_{i=0}^{N-1}2\beta\sqrt{2(1+M\alpha)}\left \|\bm{x}_{t_i} \right \|_{(\bm{\Sigma}_{m_{t_i},t_i}^{all})^{-1}}.
\]

From Section 6.2 in \cite{FedLinUCB}, $N'\leq d\ln(1+TL^2/\lambda)$.

According to Lemma \ref{per-round regret} and Lemma \ref{Sketched transform}, we have
\[
\sum_{i=0}^{N-1}\sum_{t=t_i+1}^{t_{i+1}-1}2\beta \left \|\bm{x}_t \right \|_{\bm{\widetilde{V}}_{m_t,t}^{-1}}\leq \sum_{i=0}^{N-1}\sum_{t=t_i+1}^{t_{i+1}-1}2\beta \sqrt{2(1+M\alpha)} \left \|\bm{x}_t \right \|_{(\bm{\Sigma}_{m_t,t}^{all})^{-1}}.
\]

Then
\begin{align}
Regret(T)
&=\sum_{t=1}^{T}\left ( \max_{\bm{x}\in D_t} \left \langle \bm{\theta}^*,\bm{x} \right \rangle- \left \langle \bm{\theta}^*,\bm{x}_t \right \rangle \right )\\
&\leq 2dSLM\ln(1+TL^2/\lambda)+2\beta_t\sqrt{2(1+M\alpha)}\sum_{t=1}^{T}\left \|\bm{x}_t \right \|_{(\bm{\Sigma}_{m_t,t}^{all})^{-1}}\\
&\leq 2dSLM\ln(1+TL^2/\lambda)+2\beta\sqrt{2(1+M\alpha)}\sqrt{T\sum_{t=1}^{T}\left \|\bm{x}_t \right \|_{(\bm{\Sigma}_{m_t,t}^{all})^{-1}}^2}\\
&\leq 2dSLM\ln(1+TL^2/\lambda)+2\beta\sqrt{2(1+M\alpha)}\sqrt{T\left (2d\ln\left (1+\frac{\bar{\rho}_T}{\lambda} \right )+2l\ln\left (1+\frac{TL^2}{l\lambda} \right )\right )}\\
&\leq 2dSLM\ln(1+TL^2/\lambda)+2\beta\sqrt{2(1+M\alpha)}\sqrt{2lT\ln\left (1+\frac{TL^2}{l\lambda} \right )+2dT\ln\left (1+\varepsilon_l \right )},
\end{align}

where 
\begin{align}
\beta
&=\left (\sqrt{1{+}M\alpha}{+}M\sqrt{2\alpha}\right )\left ( R\sqrt{d\log{\left (\frac{1+TL^2/(\widetilde{\alpha}\lambda)}{\delta}\right )}}{+}\sqrt{\lambda}S \right ){+}(\sqrt{R_1(M+1/\alpha)d\varepsilon_l\log (1+\varepsilon_l)T}{+}\sqrt{\lambda})S\\
&=\widetilde{O}\left((1+M\sqrt{\alpha}+\sqrt{(1/\alpha+M)\varepsilon_l})\sqrt{d}\right).
\end{align}

Therefore
\[
Regret(T)\leq \widetilde{O}\left((1+M\sqrt{\alpha}+\sqrt{(1/\alpha+M)\varepsilon_l})(\sqrt{ld}+d\ln(1+\varepsilon_l))\sqrt{(1+M\alpha)T}\right).
\]

We define $r=rank(\bm{X}_T)$ is the rank of
covariance matrix (i.e. the selected-arm sequence), setting $l\geq r$ gives $\varepsilon_l=0$, and
\[
Regret(T)\leq 2dSLM\ln(1+TL^2/\lambda)+8\beta\sqrt{1+M\alpha}\sqrt{lT\ln\left (1+\frac{TL^2}{l\lambda} \right )}=\widetilde{O}\left ( (1+M\sqrt{\alpha}) \sqrt{(1+M\alpha)ldT}\right ).
\]

Furthermore, if we set $l= r$ with $\alpha=1/M^2$ additionally, we can get $Regret(T)=\widetilde{O}(\sqrt{rdT})$.

\end{document}